\documentclass[11pt]{article}

\usepackage[utf8]{inputenc}
\usepackage{soul}
\usepackage{url,hyperref,lineno,microtype,subcaption}
\usepackage{cite}

\usepackage{lipsum}
\usepackage{authblk}

\usepackage[margin=1in]{geometry}

\usepackage[ruled]{algorithm2e}
\SetKwInput{KwInput}{Input}
\SetKwInput{KwOutput}{Output}
\usepackage{amssymb, amsmath, amsthm,mathtools}

\newtheorem{theorem}{Theorem}[section]
\newtheorem{proposition}[theorem]{Proposition}
\newtheorem{corollary}[theorem]{Corollary}
\newtheorem{lemma}[theorem]{Lemma}
\newtheorem{definition}[theorem]{Definition}
\newtheorem{assumption}{Assumption}
\theoremstyle{remark}
\newtheorem{remark}{Remark}[section]

\newtheorem{conjecture*}{Conjecture}
\theoremstyle{plain}

\newcommand{\calL}{\mathcal{L}}
\newcommand{\calX}{\mathcal{X}}
\newcommand{\calP}{\mathcal{P}}

\newcommand{\calF}{\mathcal{F}}
\newcommand{\calN}{\mathcal{N}}
\newcommand{\calE}{\mathcal{E}}
\newcommand{\calH}{\mathcal{H}}
\newcommand{\calT}{\mathcal{T}}
\newcommand{\calW}{\mathcal{W}}

\newcommand{\E}{\mathbb{E}}
\newcommand{\R}{\mathbb{R}}

\usepackage[dvipsnames]{xcolor}

\newcommand{\W}[0]{\mathcal{W}_2}
\newcommand{\fc}[2]{\frac{#1}{#2}}
\newcommand{\nb}[0]{\nabla}
\newcommand{\an}[1]{\left\langle {#1}\right\rangle}

\newcommand{\pa}[1]{\left( {#1} \right)}
\newcommand{\ve}[1]{\left\Vert {#1}\right\Vert}

\newcommand\blfootnote[1]{%
  \begingroup
  \renewcommand\thefootnote{}\footnote{#1}%
  \addtocounter{footnote}{-1}%
  \endgroup
}

\title{Convergence of flow-based generative models via proximal gradient descent in Wasserstein space}

\author[1]{Xiuyuan Cheng}
\author[1,2]{Jianfeng Lu}
\author[1]{Yixin Tan}
\author[3]{Yao Xie}

\affil[1]{{\small Department of Mathematics, Duke University}}
\affil[2]{{\small Department of Physics and Department of Chemistry, Duke University}}
\affil[3]{{\small H. Milton Stewart School of Industrial and Systems Engineering, Georgia Institute of Technology}}

\date{\blfootnote{The authors are listed alphabetically.}
\vspace{-30pt}}

\begin{document}
\maketitle

\begin{abstract}
Flow-based generative models enjoy certain advantages in computing the data generation and the likelihood, and have recently shown competitive empirical performance. Compared to the accumulating theoretical studies on related score-based diffusion models, analysis of flow-based models, which are deterministic in both forward (data-to-noise) and reverse (noise-to-data) directions, remain sparse. In this paper, we provide a theoretical guarantee of generating data distribution by a progressive flow model, the so-called JKO flow model, which implements the Jordan-Kinderleherer-Otto (JKO) scheme in a normalizing flow network. Leveraging the exponential convergence of the proximal gradient descent (GD) in Wasserstein space, we prove the Kullback-Leibler (KL) guarantee of data generation by a JKO flow model to be  $O(\varepsilon^2)$ when using $N \lesssim \log (1/\varepsilon)$ many JKO steps ($N$ Residual Blocks in the flow) where $\varepsilon $ is the error in the per-step first-order condition. The assumption on data density is merely a finite second moment, and the theory extends to data distributions without density and when there are inversion errors in the reverse process where we obtain KL-$\mathcal{W}_2$ mixed error guarantees. The non-asymptotic convergence rate of the JKO-type $\mathcal{W}_2$-proximal GD is proved for a general class of convex objective functionals that includes the KL divergence as a special case, which can be of independent interest.
The analysis framework can extend to other first-order Wasserstein optimization schemes applied to flow-based generative models.
\end{abstract}

\section{Introduction}

Generative models, from generative adversarial networks (GAN) \cite{GAN,WassersteinGAN,CGAN} and variational auto-encoder (VAE) \cite{kingma2013auto,VAE_review}
to normalizing flow \cite{nflow_review},
have achieved many successes in applications and have become a central topic in deep learning.
More recently, diffusion models \cite{song2019generative,ho2020denoising,song2021score} and closely related flow-based models \cite{lipman2023flow,albergo2023building,albergo2023stochastic,fan2022variational,xu2022jko} have drawn much research attention, given their state-of-the-art performance in image generations.
Compared to score-based diffusion models, which are designed for sampling,
flow models have certain advantages due to their direct capability in estimating likelihood, a basis for statistical inference. 
However, despite the empirical successes, the theoretical understanding and guarantees for flow-based generative models remain limited.

In this paper, we provide a theoretical guarantee of generating data distribution by a ``progressive'' flow model, 
mainly following the JKO flow model in \cite{xu2022jko} but similar models have been proposed in, e.g., \cite{alvarez2022optimizing,mokrov2021large,vidal2023taming}.
We prove the exponential convergence rate of such flow models in both (data-to-noise and noise-to-data) directions.  
Below, we give an overview of the main results. 
We provide a brief introduction of the most related types of flow-based models, particularly the progressive one, in Section \ref{subsec:intro-flow}.
A more complete
literature survey can be found in Section \ref{subsec:literature}.

An abundance of theoretical works has provided the generation guarantee of score-based diffusion models \cite{lee2022convergence, lee2023convergence, chen2022sampling, de2022convergence, benton2024nearly, chen2023improved,pedrotti2024improved}.
In comparison, the theoretical study of flow-based generative models is much less developed.
Most recent works on the topic focused on the generation guarantee of the Ordinary Differential Equation (ODE) 
reverse process (deterministic sampler) once a score-based model is trained from the {\it forward} Stochastic Differential Equation (SDE) 
diffusion process \cite{chen2023restoration,chen2024probability,li2024towards}.
For generative models which are flow-based in the forward process,
generation guarantee for flow-matching models under continuous-time formulation was shown in \cite{albergo2023building} under $\W$, 
and  in \cite{albergo2023stochastic} under the Kullback–Leibler (KL) divergence by incorporating additional SDE diffusion.
The current paper focuses on obtaining the theoretical guarantee of the JKO flow model \cite{xu2022jko}, which is progressively trained over the Residual Blocks (steps) and generates a discrete-time flow in both forward and reverse directions. The mathematical formulation of the JKO flow is summarized in Section \ref{sec:setup}, where we introduce needed theoretical assumptions on the learning procedure.

Our analysis is based on first proving the convergence of the forward process (the JKO scheme by flow network), which can be viewed as an approximate proximal Gradient Decent (GD) in the Wasserstein-2 space to minimize $G(\rho)$, a functional on the space of probability distributions. 
While the convergence analyses of Wasserstein GD and proximal GD have appeared previously in literature \cite{salim2020wasserstein,kent2021modified}, our setup differs in several ways, primarily in that we consider the JKO scheme, which is a ``fully-backward'' discrete-time GD.
For the $N$ step discrete-time proximal GD, which produces a sequence of transported distributions $p_n$, we prove the convergence of 
both $\W(p_n, q)$ and the objective gap $G(p_n) - G(q)$ at an exponential rate, where $q$ is the global minimum of $G$ (Theorem \ref{thm:N-step-forward-no-inv-error}). The convergence applies to a general class of (strongly) convex $G$ that includes the KL divergence ${\rm KL}(p || q)$ as a special case. 
This result echos the classical proximal GD convergence in vector space where one expects an exponential convergence rate for strongly convex minimizing objectives. 
While exponential convergence is a natural result from the point of view of gradient flow, this convergence result of JKO-type $\W$-proximal GD did not appear in previous literature to the authors' best knowledge and can be of independent interest.

After obtaining a small $G(p_n) = {\rm KL}(p_n || q)$ from the convergence of the forward process, we directly obtain the KL guarantee of the generated density from the data density by the invertibility of the flow and 
the data processing inequality, 
and this implies the total variation (TV) guarantee (Corollary \ref{cor:KL-P-P2r}). 
The requirement for data distribution is to have a finite second moment and a density (with respect to the Lebesgue measure). 
The TV and KL guarantees are of $O(\varepsilon)$ and $O(\varepsilon^2)$, respectively, where $\varepsilon$ is the bound for the magnitude of the Wasserstein (sub-)gradient of the loss function (hence error in the first order condition) in each of the $N$ JKO steps (Assumption \ref{assump:1st-order-condition-error}), and the process achieves the error bound in $N \lesssim \log (1/\varepsilon)$ many steps (each step is a Residual Block).

To handle the situation when the data distribution only has a finite second moment but no density, we apply a short-time initial diffusion and start the forward process from the smoothed density. This short-time diffusion was adopted in practice and prior theoretical works. 
We then obtain KL and TV guarantee to generate the smoothed density, which can be made arbitrarily close to the data distribution in $\W$ when the initial diffusion time duration tends to zero (Corollary \ref{cor:KL-mixed-P2}). The above results are obtained when the reverse process is computed exactly with no inversion error. 
Our analysis can also extend to the case of small inversion error by proving a $\W$-guarantee between the generated density from the exact reverse process and that from the actual computed one. Theoretically, the $\W$-error can be made $O(\varepsilon)$ or smaller assuming that the inversion error can be made $O(\varepsilon^{\alpha})$ for some exponent $\alpha$ (Corollary \ref{cor:mixed-bound-inv}).

\begin{figure*}[t]
\centering
\includegraphics[height=.275\linewidth]{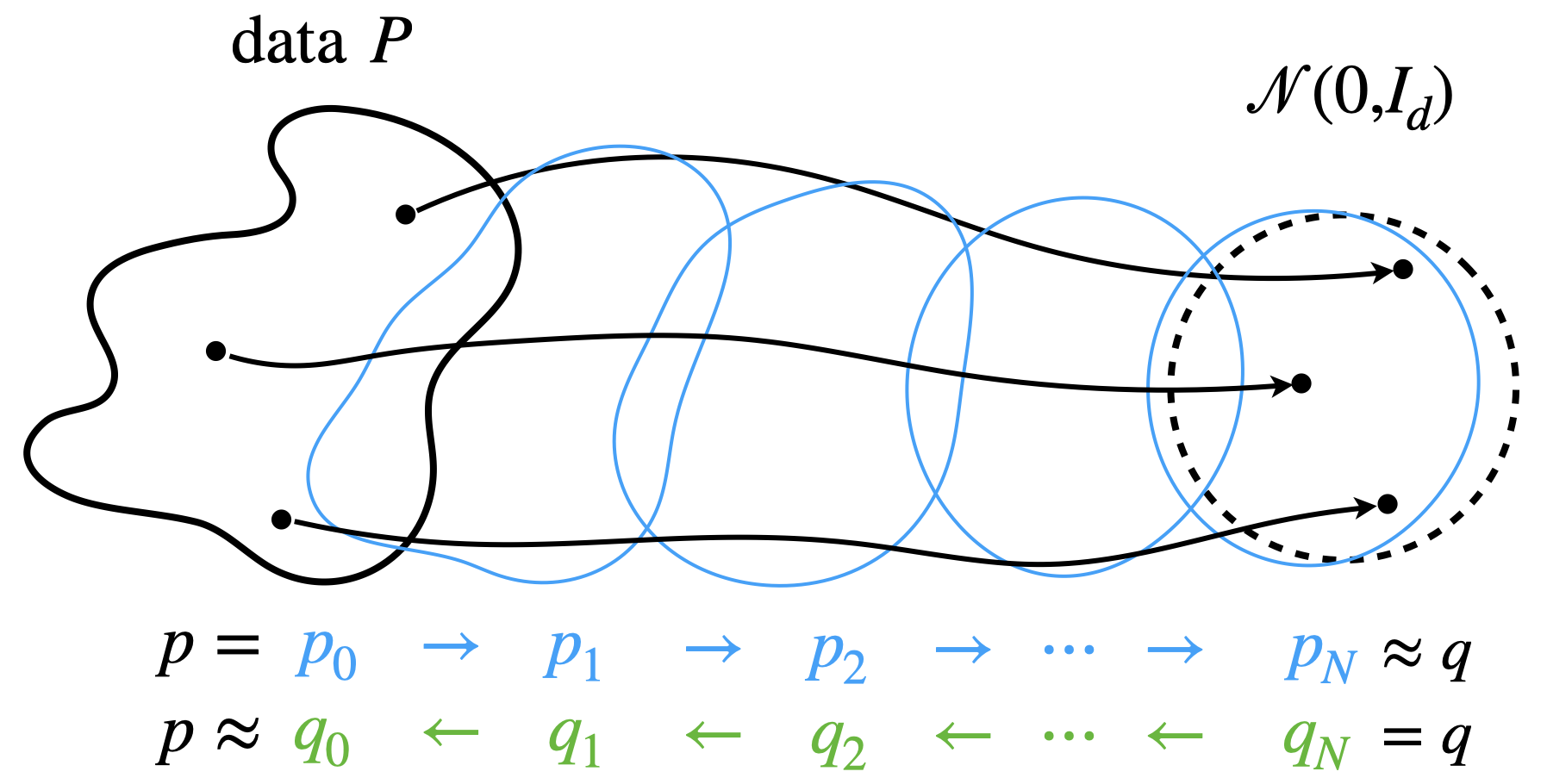} 
\caption{
The arrows indicate the forward-time flow from data distribution $P$ to normal distribution $q$.
The forward and reverse processes  \eqref{eq:fwd-bwd-process}
consist of the sequence of transported densities at discrete time stamps.
}
\label{fig:pn}
\end{figure*}

\subsection{Normalizing Flow models}\label{subsec:intro-flow}

\paragraph{Normalizing flow.}
Normalizing flow is a class of deep generative models for efficient sampling and density estimation. 
Compared to diffusion models, Continuous Normalizing Flow (CNF) models \cite{nflow_review} appear earlier in the generative model literature. Largely speaking, CNF  models fall into two categories: discrete-time and continuous-time.
The discrete-time CNF models adopt the structure of a Residual Network (ResNet) \cite{he2016deep} and typically consist of a sequence of mappings:
\begin{equation}\label{eq:resnet}
x_{l} = x_{l-1} +  f_l (x_{l-1}), \quad l = 1, \cdots, L,
\end{equation}
where $f_l$ is the neural network mapping parameterized by the $l$-th ``Residual Block'', and $x_l$ is the output of the  $l$-th block. 
Continuous-time CNFs are implemented under the neural ODE framework \cite{chen2018neural},
where the neural network features $x(t)$ is computed by integrating an ODE
\begin{equation}\label{eq:flownet}
\dot{x}(t)= v(x(t), t), \quad t \in [0,T],
\end{equation}
and $v_t(x) = v(x,t)$ is parametrized by a neural ODE network. 
The discrete-time CNF \eqref{eq:resnet} can be viewed as computing the numerical integration of the neural ODE \eqref{eq:flownet} on a sequence of time stamps via the forward Euler scheme. 

In both categories, a CNF model computes a deterministic transport from the data distribution towards a target distribution $q$ typically normal, $q = \calN(0, I_d)$, per the name ``normalizing.'' The forward time flow is illustrated in Figure \ref{fig:pn}.
Taking the continuous-time formulation \eqref{eq:flownet}, 
let $P$ be the data distribution with density $p$, $x(0) \sim p$,
and denote by $p_t (x) = p(x,t)$ the probability density of $x(t)$.
Then $p_t$ solve the continuity equation (CE)
\begin{equation}\label{eq:liouville}
\partial_t p_t + \nabla \cdot (p_t  v_t) = 0,
\end{equation}
from $p_0 = p$. 
If the algorithm can find a  $v_t$ such that $p_T$ at some time $T$ is close to $q$, then one would expect the reverse-time flow from $t=T$ to $t=0$ to transport from $q$ to a distribution close to $p$.
Note that in the continuous-time flow, invertibility is presumed since the neural ODE can be integrated in two directions of time alike.
 For discrete-time flow \eqref{eq:resnet}, invertibility needs to be ensured either by special designs of the neural network layer type  \cite{dinh2014nice,dinh2016density,kingma2018glow},  
 or by regularization techniques such as spectral normalization \cite{iResnet}
 or transport cost regularization \cite{onken2021otflow,xu2022invertible}.
 
 A notable advantage of the flow model is the computation of the likelihood. 
 For discrete-time flow \eqref{eq:resnet}, this involves the computation of the log-determinant of the Jacobian of $ f_l$.
 For continuous-time flow \eqref{eq:flownet}, this is by the relation
 \[\log p_t( x(t)) - \log p_s( x(s))  =  - \int_{s}^{t} \nabla \cdot v( x(\tau), \tau) d\tau,\] 
 which involves the time-integration of the trace of the Jacobian of $ v_t$ \cite{grathwohl2018ffjord}. While these computations may encounter challenges in high dimensions, the ability to evaluate the (log) likelihood is fundamentally useful; in particular, it allows for evaluating the maximum likelihood training objective on finite samples. This property is also adopted in the deterministic reverse process in diffusion models \cite{song2021score}, called the ``probability flow ODE'' (see more in Section \ref{subsec:lit-sbdm}), so the likelihood can be evaluated once a forward diffusion model has been trained. 
 
\paragraph{Progressive flow models.} Another line of works, developed around the same time as diffusion models,
 explored the variational form of normalizing flow as a Wasserstein gradient flow and proposed the so-called {\it progressive} training of the flow model.
The progressive training of ResNet, i.e., training block-wise by a per-block variational loss, was proposed by \cite{johnson2019framework} at an earlier time under the GAN framework.
Later on, the Jordan-Kinderleherer-Otto (JKO) scheme, as a time-discretized Wasserstein gradient flow (see more in Section \ref{sub:pre-sde-jko}), was explored in several flow-based generative models:
\cite{alvarez2022optimizing,mokrov2021large} implemented the JKO scheme using input convex neural networks \cite{amos2017input};
\cite{fan2022variational} proposed a  forward progressive flow from noise to data, showing empirical success in generating high-dimensional real datasets;
\cite{xu2022jko} developed  the JKO flow model under the invertible continuous-time CNF framework, 
achieving competitive generating performance on high-dimensional real datasets 
at a significantly reduced computational and memory cost from previous CNF models;
an independent concurrent work \cite{vidal2023taming} proposed a block-wise JKO flow model utilizing the framework of \cite{onken2021otflow}.
Many other flow models related to diffusion and Optimal Transport (OT) exist in the literature; see more in Section \ref{subsec:lit-more-flows}.
Our theoretical analysis will focus on the progressive flow models, and we primarily follow the invertible flow framework in \cite{xu2022jko}.
 
 To be more specific, a progressive flow model represents the flow on  $[0,T]$ as the composition of $N$ sub-flow models
  where each one computes the flow on a sub-interval $[t_{n-1}, t_n]$, $n=1, \dots, N$. 
The training is ``progressive'', meaning that at one time, only one sub-model is trained,
and the training of the $n$-th sub-model is conducted once the previous $n-1$ sub-models are trained and fixed. 
The progressive block-wise training is in contrast of the end-to-end training, where the flow on $[0,T]$ (or $N$ Residual Blocks) is trained simultaneously by a single objective.
The sub-flow model on $[t_{n-1}, t_n]$ can take different forms,
e.g., a ResNet block or a continuous-time neural ODE, 
and the $N$ sub-intervals always provide a time-discretization of the flow. 
In this context, we call the sub-model on the $n$-th sub-interval a ``Residual Block''.

\subsection{Additional related works}\label{subsec:literature}

\subsubsection{Score-based diffusion models}\label{subsec:lit-sbdm}

In score-based diffusion models, the algorithm first simulates a forward process, which is a (time-discretized) SDE,
from which a score function parameterized as a neural network is trained.
The reverse (SDE or ODE) process is simulated using the learned score model to generate data samples.

\paragraph{SDE in diffusion models}

As a typical example, in the {\it variance preserving} Denoising Diffusion Probabilistic Modeling (DDPM) process \cite{ho2020denoising,song2021score}, the forward process produces a sequence of $X_n$,
\begin{equation}\label{eq:VP-DDPM-1}
X_n = \sqrt{1-\beta_n } X_{n-1} + \sqrt{ \beta_n } Z_{n-1}, 
\quad n = 1, \cdots, N,
\end{equation}
where $Z_{n} \sim \calN(0, I_d)$ i.i.d.
and 
$X_0 \sim P $ is drawn from data distribution.
With large $N$, the continuum limit of the discrete dynamic \eqref{eq:VP-DDPM-1} is a continuous-time SDE, 
\begin{equation}\label{eq:VP-DDPM=SDE}
dX_t = -\frac{1}{2}\beta(t)  X_t dt + \sqrt{\beta (t)} dW_t,
\quad  t\in [0, T], 
\end{equation}
where $\beta(t) > 0$ is a function and
$W_{t}$ is a standard Wiener process (Brownian motion).
Since $\beta(t)$ in \eqref{eq:VP-DDPM=SDE} corresponds to a time reparametrization of $t$, after changing the time ($t \mapsto \int_0^t \beta(s)/2 ds$), \eqref{eq:VP-DDPM=SDE} becomes  the following SDE 
\begin{equation}\label{eq:OU-SDE}
    dX_t = - X_t dt + \sqrt{2} dW_t,
\end{equation}
which is the Ornstein-Uhlenbeck (OU) process in $\R^d$.
We consider the time-parametrization in \eqref{eq:OU-SDE} for exhibition simplicity.
More generally, one can consider a diffusion process 
 \begin{equation}\label{eq:diffusion-sde-2}
 dX_{t}= - \nabla V(X_{t} )\,dt+ \sqrt{2}\,dW_{t}, \quad X_0 \sim P,
 \end{equation}
and the OU process is a special case with $V(x) = \| x\|^2/2$.

We denote by $\rho_t = {\calL}_t (P)$ the marginal distribution of $X_t$ for $ t >0$. 
The time evolution of $\rho_t$ is described by the Fokker–Planck Equation (FPE)  written as
\begin{equation}\label{eq:FPE}
\partial_{t}\rho_t = \nabla\cdot(\rho_t \nabla V + \nabla \rho_t).
\end{equation}

\paragraph{Forward and reverse processes}
When simulating the forward  process, the diffusion models train a neural network to learn the score function 
$s_t( x) : = \nabla \log \rho_t$
by score matching \cite{hyvarinen2005estimation,vincent2011connection}. The training objective can be expressed as the mean-squared error defined as $\int_0^T \int \| \hat s_t (x) - s_t (x) \|^2 \rho_t (x) dx dt $, which facilitates training and is scalable to high dimension data such as images (in the original pixel space). 

Once the neural-network score function $\hat s_t$ is learned, 
the algorithm simulates a reverse-time SDE $\tilde X_t$ (with time discretization in practice) \cite{song2021score},
such that from $\tilde X_T \sim \calN(0,I)$ the distribution of $\tilde X_0$ is expected to the close to the data distribution $P$. 
It has also been proposed in \cite{song2021score} to compute the reverse process by integrating the following ODE reverse in time
\begin{equation}\label{eq:prob-flow-ode}
\dot{\tilde x} (t) = - \nabla V( \tilde{x}(t)) - s_t( \tilde x(t) ),
\end{equation}
and \eqref{eq:prob-flow-ode} was called the ``probability flow ODE.'' The validity of this ODE reverse process can be justified by the observation that the
CE \eqref{eq:liouville} and FPE \eqref{eq:FPE} are the same when setting $v_t(x) =  - (\nabla V(x) + s_t(x))$.
This equivalence between density evolutions by SDE and ODE 
has been known in the literature of diffusion processes and solving FPE, dating back to the 90s \cite{degond1990deterministic,degond1989weighted}.

\subsubsection{Flow models related to diffusion and OT}\label{subsec:lit-more-flows}

\paragraph{Flow-matching models}
After diffusion models gained popularity, several flow-based models (in the reverse and forward directions) closely related to the diffusion model emerged. 
In particular, the Flow-Matching ODE model was proposed in \cite{lipman2023flow} using the formulation conditional probability paths, where a neural ODE parameterized $\hat v(x,t)$ is trained to match a velocity field $v(x,t)$ whose corresponding CE \eqref{eq:liouville} can evolve the density $p_t$ towards normality.
The algorithm can adopt diffusion paths, where the CE will equal the density evolution equation \eqref{eq:FPE} of an SDE forward process, as well as non-diffusion paths.
A similar approach was developed under the ``stochastic interpolant'' framework in \cite{albergo2023building}, where the terminal distribution $q$ can be arbitrary (not necessarily the normal distribution) and only accessible via finite samples. These models train a continuous-time CNF by minimizing a ``matching'' objective instead of the maximum likelihood objective as in \cite{grathwohl2018ffjord}, thus avoiding the computational challenges of the latter.

\paragraph{Optimal Transport flows}
Apart from diffusion models and Wasserstein gradient flow, 
Wasserstein distance and OT have inspired another line of works on flow models where the Wasserstein distance, or a certain form of transport cost, is used to regularize the flow model and to compute the OT map between two distributions. Transport cost regularization of neural network models was suggested in several places: 
\cite{ruthotto2020machine} provided a general framework for solving high-dimensional mean-field games (MFG) and control problems, \cite{finlay2020train} proposed a kinetic regularization aiming to stabilize neural ODE training, 
\cite{onken2021otflow} and \cite{xu2022invertible} developed the transport regularization in CNF and invertible ResNet, respectively,
and \cite{huang2023bridging} applied to MFG and flow models. 
Other works developed flow models to compute the optimal coupling or the optimal transport between two distributions.
For example, Rectified Flow \cite{liu2022rectified} proposed an iterative method to adjust
the flow towards the optimal coupling. 
The method is closely related to the stochastic interpolant approach \cite{albergo2023building} which, in principle, can solve the OT trajectory if the interpolant map can be optimized. A flow model to compute the dynamic OT between two high dimensional distributions from data samples was proposed in \cite{xu2023computing} by refining the flow using the transport cost from a proper initialization. Despite the wealth of methodology developments and empirical results, the theoretical guarantees of these flow models are yet to be developed.

\subsubsection{Theoretical guarantees of generative models}

\paragraph{Approximation and estimation of GAN}
On theoretical guarantees of generative models, earlier works focused on the approximation and estimation analysis under the GAN framework. 
The expressiveness of a deep neural network to approximate high dimensional distributions was established in a series of works, e.g.,  \cite{lee2017ability,lu2020universal,perekrestenko2021high,yang2022capacity}, among others.  
The neural network architectures in these universal approximation results are typically feed-forward, like the generator network (G-net) proposed in the original GAN.
The approximation and estimation of the discriminator network (D-net) in GAN were studied in \cite{cheng2022classification}, and the problem can be cast and analyzed as the learning of distribution divergences in high dimension \cite{sreekumar2022neural}.
Convergence analysis of GAN was studied in several places, e.g., \cite{huang2022error}.

\paragraph{Guarantees of diffusion models}
An earlier work \cite{tzen2019theoretical} studied
the expressiveness of a generative model using a latent diffusion process and proved guarantees for sampling and inference; however, the approach only involves a forward process and differs from the recent diffusion models. 
Motivated by the prevailing empirical success of score-based diffusion models, recent theoretical works centralized on the generation guarantee of such models using both SDE and ODE samplers, i.e., the reverse process.

For the SDE reverse process, the likelihood guarantee of the score-based diffusion model was first derived in \cite{song2021maximum} without time discretization. 
Taking into account the time discretization, which significantly influences the efficiency in practice,
a series of theoretical studies have established polynomial convergence bounds for such models~\cite{lee2022convergence, lee2023convergence, chen2022sampling, de2022convergence, benton2024nearly, chen2023improved, pedrotti2024improved}.
In particular, \cite{lee2022convergence} were the first to attain polynomial convergence without succumbing to the curse of dimensionality, although this required a log-Sobolev Inequality on the data distribution. For a general data distribution,~\cite{chen2022sampling} achieved polynomial error bounds in Total Variation (TV) distance under the Lipschitz assumption, leveraging Girsanov's theorem. In parallel,~\cite{lee2023convergence} derived similar polynomial convergence bounds, employing a technique for converting $L^\infty$-
accurate score estimates into $L^2$-accurate score estimation. Further advancements by~\cite{chen2023improved} established a
more refined bound, reducing the requirement of smoothness of data distribution. Most recently,
\cite{pedrotti2024improved} improved the convergence rates under mild assumptions by introducing prediction-correction,
and ~\cite{benton2024nearly} established the first convergence bounds for diffusion models, which are linear in the data dimension (up to logarithmic factors) without requiring any smoothness of the data distribution.

\paragraph{Guarantees of ODE flows} 
Within the studies of score-based diffusion models (note that the forward process is always SDE), theoretical findings for the ODE reverse process are relatively fewer. 
To the best of our knowledge,
\cite{chen2023restoration} established the first non-asymptotic polynomial convergence rate where the error bound involves an exponential factor in the flow time;
\cite{chen2024probability} provided the first polynomial-time convergence guarantees for the probability flow ODE implementation with a corrector step.
Recently,~\cite{li2024towards} established bounds for both deterministic (ODE) and non-deterministic (SDE) samplers under certain additional assumptions on learning the score.
The analysis is done by directly tracking the density ratio between the law of the diffusion process and that of the generated process in discrete time, leading to various non-asymptotic convergence rates.

Compared to score-based diffusion models, the guarantees of flow models (in both forward and reverse processes) are significantly less developed. We are aware of two recent works:
The error bounds for the flow-matching model \cite{albergo2023building} were proved in \cite{benton2024error} and applied to probability flow ODE in score-based diffusion models;
for neural ODE models trained by likelihood maximization (the framework in \cite{grathwohl2018ffjord}), \cite{marzouk2023distribution} proved non-parametric statistical convergence rates to learn a distribution from data. Both works used a continuous-time formulation, and the flow models therein are trained end-to-end.
Compared to end-to-end training, progressive flow models can offer advantages in training efficiency and accuracy, in addition to other advantages such as smaller model complexity. 
For the analysis, the formulation of progressive flow models is variational and time-discretized in nature. Theoretical studies of time-discretized ODE flow models in both forward and reverse directions remain rudimentary.

\subsubsection{Optimization in Wasserstein space}

Continuing the classical literature in optimization and information geometry, several recent works established a convergence guarantee of first-order optimization in probability space in various contexts, leveraging the connection to the Wasserstein gradient flow. These analyses can potentially be leveraged under the theoretical framework of this paper to develop new (progressive) flow models as well as theoretical guarantees of generative models.

\paragraph{Optimization in probability distribution space}
Convergence and rate analysis for first-order methods for vector-space optimization, primarily gradient descent and stochastic gradient descent
--- sometimes referred to as the Sample Average Approximation (SAA) approach --- for convex and strongly convex problems have been established in the original works \cite{NemirovskiYudin83,nemirovski2009robust}, and extended in various contexts in subsequent papers.
Optimization in the space of probability distributions (which forms a manifold) naturally arises in many learning problems and has become an important field of study in statistics and machine learning. 
In particular, the seminal work of Amari \cite{amari2016information,amari2008information} introduced information geometry emerging from studies of a manifold of probability distributions. 
It includes convex analysis and its duality as a special but important component; however, the line of work did not develop error analysis or convergence rates for algorithms on the probabilistic manifold.
More recently, a Frank-Wolfe procedure in probability space was proposed in \cite{kent2021modified} motivated by applications in nonparametric estimation and was shown to converge exponentially fast under general mild assumptions on the objective functional.

\paragraph{Wasserstein proximal gradient descent} 
The landmark work \cite{jordan1998variational} showed 
the solution to the Fokker Planck equation as the gradient flow of the KL divergence under the $\W$-distance. 
The proof in \cite{jordan1998variational} employed a time discretization of the gradient flow now recognized as the JKO scheme. 
Making a connection between Langevin Monte Carlo and Wasserstein gradient flow, 
\cite{bernton2018langevin} proposed a proximal version of the Unadjusted Langevin Algorithm 
corresponding to a splitting scheme of the discrete Wasserstein GD 
and derived non-asymptotic convergence analysis. 
To analyze the convergence of discrete-time $\W$-gradient flow, \cite{salim2020wasserstein} introduced a Forward-Backward time discretization in the proximal Wasserstein GD
 and proved convergence guarantees akin to the GD algorithm in Euclidean spaces. 
 We comment on the difference between \cite{salim2020wasserstein} and our scheme in more detail later in Remark \ref{remark-forward-backward}.

The JKO scheme also inspired recent studies in variational inferences (VI). In the context of Gaussian VI, 
  \cite{lambert2022variational} proposed gradient flow of the KL divergence on the Bures-Wasserstein (BW) space,
  namely the space of Gaussian distributions on $\R^d$ endowed with the $\W$-distance.
  The algorithm enjoys the explicit solution of the JKO scheme in the BW space, and convergence of the continuous-time gradient flow was proved. In a follow-up work \cite{diao2023forward}, the forward-backward splitting was adopted in the proximal Wasserstein GD in the BW space,
  leading to convergence guarantees of the discrete-time GD to first-order stationary solutions. 
  The closed-form solution of the JKO operator only applies to the BW space, while the JKO flow network tries to learn a transport map to solve the JKO scheme in each step, leveraging the expressiveness of neural networks.  
  Theoretically, we consider distributions with finite second moments in this work.

\subsection{Notations}
Throughout the paper, we consider distributions over $\calX$ and the domain $\calX = \R^d$.
We denote by $\calP_2$, meaning $\calP_2(\R^d)$, the space of probability distributions on $\R^d$ that has finite second moment.
Specifically, for a distribution $P$, define $M_2(P) := \int_{\R^d} \| x \|^2 dP(x)$. 
When $P$ has a density (with respect to the Lebesgue measure $dx$), we also write $M_2(P)$ as $M_2(p)$.
Then $\calP_2 = \{ P \text{ on $\R^d$}, \, s.t., M_2(P) < \infty \}$.
We denote by $\calP_2^r$ the distributions in $\calP_2$ that have densities, namely
$\calP_2^r = \{ P \in \calP_2, \, P \ll dx \}$. We also say a density $p \in \calP_2^r$ when $dP(x) = p(x)dx$ is in $\calP_2^r$.
In this paper, we consider distributions that have densities in most places. When there is no confusion, we use the density $p$ to stand for both the density and the distribution $dP(x) = p(x)dx$,  e.g., we say that a random variable $X \sim p$.

Given a (measurable) map $v: \R^d \to \R^d$ and $P$ a distribution on $\R^d$, 
its $L^2$ norm is denoted as  $ \| v\|_P := ( \int_{\R^d} \|v (x) \|^2 dP(x) )^{1/2}$.
When $P$ has density $p$, we also denote it as $\| v \|_p $.
For $P \in \calP_2$, we denote by $L^2(P)$ (and also by $L^2(p)$ when $P$  has density $p$) 
the $L^2$ space of vector fields, that is, $L^2(P) := \{ v : \R^d \to \R^d, \, \| v \|_P <\infty \}$.
For $ u, v \in L^2(P)$, define $\langle u, v \rangle_P : = \int_{\R^d} u(x)^T v(x) dP(x)$, which is also denoted as $\langle u, v \rangle_p$ when $p$ is the density.
The notation ${\rm I_d}$ stands for the identity map, which is always in $L^2(P)$ for $P \in \calP_2$.
For $T:\R^d \to \R^d$, 
the {\it pushforward} of a distribution $P$ is denoted as $T_\# P$, 
such that 
$T_\# P(A) = P( T^{-1}(A))$ for any measurable set $A$.
When $P$ has density $p$ and $T_\# P$ also has a density, 
we also denote by $T_\# p$ the density of  $T_\# P$.
For two maps $S, T: \R^d \to \R^d$, $S \circ T$ 
is the function composition.

\section{Preliminaries}

\subsection{Wasserstein-2 distance and optimal transport}

We first review the definitions of the Wasserstein-2 distance and optimal transport (OT) map, which are connected by the Brenier Theorem (see, e.g., \cite[Section 6.2.3]{ambrosio2005gradient}).

Given two distributions $\mu, \nu \in \calP_2$, the Wasserstein-2 distance $\W( \mu, \nu)$ is defined as
\begin{equation}\label{eq:ot}
    \W^2( \mu, \nu ) := \inf_{\pi \in \Pi ( \mu, \nu)}
        \int_{ \R^d \times \R^d} \ve{x-y}^2 d\pi(x,y),
\end{equation}
where $\Pi (\mu, \nu)$ denotes the family of all joint distributions with $\mu$ and $\nu$ as marginal distributions.
When $P$ and $Q$ are in $\calP_2^r$ and  have densities $p$ and $q$ respectively, we also denote $\W(P,Q)$ as $\W(p,q)$. 
When at least one of $\mu$ and $\nu$ has density, we have the Brenier Theorem, which allows us to define the optimal transport (OT) map from $\mu$ to $\nu$.

\begin{theorem}[Brenier Theorem]\label{thm:brenier}
Let $\mu \in \calP_2^r$ and $ \nu \in \calP_2$. Then
\begin{itemize}
\item[(i)]
There exists a unique minimizer $\pi$ of~\eqref{eq:ot}, which is characterized by a uniquely determined $\mu$-$a.e.$ map $T_\mu^\nu: \R^d \rightarrow \R^d $ such that $\pi = ( {\rm I_d}, T_\mu^\nu)_\# \mu$, where $( {\rm I_d}, T_\mu^\nu)$ maps $(x,y)$ to $(x, T_\mu^\nu (y))$. 
Moreover, there exists a convex function $\varphi: \R^d \to \R$ such that $T_\mu^\nu = \nb \varphi$ $\mu$-a.e.

\item[(ii)] 
The minimum of \eqref{eq:ot} equals that of the Monge problem, namely
\[
\W^2( \mu, \nu) = \inf_{T: \R^d \to \R^d, \, T_\# \mu = \nu }
    \int \| x -T(x) \|^2 d\mu(x).
\]

\item[(iii)] If moreover $\nu \in \calP_2^r$, then we also have the OT map  $T_\nu^\mu$ defined $\nu$-a.e., and  $T_\nu^\mu \circ T_\mu^\nu = {\rm I_d}$ $\mu$-a.e.,  $ T_\mu^\nu \circ T_\nu^\mu  = {\rm I_d}$ $\nu$-a.e.
\end{itemize}
\end{theorem}

In most places in our analysis, we will consider the OT between $\mu$ and $\nu$ both in $\calP_2^r$, and we will frequently use the Brenier Theorem (iii) to obtain the pair of OT maps which are inverse of each other in the a.e. sense.

\subsection{Differential and convexity of functionals on $\calP_2$}\label{subsec:prelim-cal-P2}

Consider a proper lower semi-continuous functional $\phi: \calP_2 \to (-\infty, \infty]$ 
and we denote the domain to be  ${\rm Dom}(\phi) = \{ \mu \in \calP_2,\, \phi( \mu) < \infty \} $.
The subdifferential of $\phi$ was defined in the Fr\'echet sense, see, e.g., Definition 10.1.1 of \cite{ambrosio2005gradient}. We recall the definition of strong subdifferential as below.

\begin{definition}[Strong subdifferential]\label{def:strong-subdiff}
Given $\mu \in \calP_2$, 
a vector field $\xi \in L^2(\mu) $  is a strong (Fr\'echet) subdifferential of $\phi$ at $\mu$ if for $v \in L^2(\mu)$,
\[
\phi( ( {\rm I_d} + v )_\# \mu) - \phi( \mu) \ge \langle \xi, v \rangle_\mu + o( \| v \|_\mu).
\]
We denote by $\partial_{\W} \phi (\mu)$ the set of strong Fréchet subdifferentials of $\phi$ at $\mu$ (which may be empty).
\end{definition}

There can be different ways to introduce convexity of functions on $\calP_2$. The most common way is the convexity along geodesics, also known as ``displacement convexity.'' In our analysis, we technically need the notation of convexity {\it along generalized geodesics} (a.g.g.), which is stronger than geodesic convexity. 
In short, displacement convexity is along the geodesic from $\mu_1$ to $\mu_2$, which,
in the simple case where there is a unique OT map $T_1^2$ from $\mu_1$ to $\mu_2$, 
is defined using interpolation $(1-t){\rm I_d} + t T_1^2$ for $t \in [0,1]$.
In contrast, convexity a.g.g. involves a third distribution $\nu$ and is defined using interpolation of the two OT maps from $\nu$ to $\mu_1$ and $\mu_2$ respectively. 

Specifically,  let $\nu \in \calP_2^r$, $\mu_i \in \calP_2$, $i=1,2$, 
and let $T_\nu^i$ be the OT map from $\nu$ to $\mu_i$ respectively. 
A  {\it general geodesic} joining $\mu_1$ to $\mu_2$ (with base $\nu$) is a curve of type
\begin{equation}\label{eq:def-mut-1-2}
\mu_t^{1 \to 2} : = ( (1-t) T_\nu^1 + t T_\nu^2 )_\# \nu, \quad t \in [0,1].
\end{equation}
\begin{definition}[Convexity along generalized geodesics]\label{def:convex-agg}
For $\lambda \ge 0$, a functional $\phi$ on $\calP_2$ is said to be $\lambda$-convex along generalized geodesics (a.g.g.) if for any $\nu \in \calP_2^r$ and $\mu_1, \mu_2 \in \calP_2$
and $\forall t \in [0,1]$,
\begin{equation}
    \phi( \mu_t^{1\to 2} ) \le (1-t) \phi( \mu_1 ) + t \phi( \mu_2) - 
    \frac{\lambda }{2} t(1-t) \calW_\nu^2(\mu_1, \mu_2),
\end{equation}
where $\mu_t^{1 \to 2} $ is as in \eqref{eq:def-mut-1-2} and 
\begin{equation}
    \calW_\nu^2(\mu_1, \mu_2) : = \int_{\R^d} \| T_\nu^1(x) - T_\nu^2(x) \|^2 d\nu(x)
    \ge \W^2(\mu_1, \mu_2).
\end{equation}
\end{definition}
Note that the definition implies the following property which is useful in our analysis, $\forall t \in [0,1]$,
\begin{equation}\label{eq:convex-agg-2}
    \phi( \mu_t^{1 \to 2} ) \le (1-t) \phi( \mu_1 ) + t \phi( \mu_2) - 
    \frac{\lambda }{2} t(1-t) \calW_2^2(\mu_1, \mu_2).
\end{equation}
The definition of convexity a.g.g. in \cite[Section 9.2]{ambrosio2005gradient} is for the more general case when $\nu$ may not have density and the OT maps from $\nu$ to $\mu_i$ need to be replaced with optimal plans, and then the generalized geodesics may not be unique. In this paper, we only consider the case where $\nu$ has a density, so we simplify the definition, see \cite[Remark 9.2.3]{ambrosio2005gradient} (and make it slightly weaker, but there is no harm for our purpose). 

We also note that many functionals $\phi(\mu)$ on $\calP_2$ that are geodesically convex actually also satisfy the convexity a.g.g. in Definition \ref{def:convex-agg}. Examples include
$\phi(\mu) = \int V(x) d\mu(x)$ with $V$ convex on $\R^d$, 
$\phi(\mu) = \iint W(x_1, x_2) d\mu(x_1) d\mu(x_2)$ with convex $W$,
and
 $\phi(\rho) = \int F(\rho(x)) dx$, $\rho$ being the density, where $F$ is convex on $[0, \infty)$. (In these examples, $V$, $W$, and $F$ need to satisfy other technical conditions.) The last example includes negative entropy as a special case, where $F(x) = x \log x$. 
The primary case for our work is when $\phi$ is the KL divergence, which will be discussed in more detail in Section \ref{subsec:general-conditions}. We refer to \cite[Section 9]{ambrosio2005gradient} for other examples and detailed discussions of convex a.g.g.{} functionals.

\subsection{JKO scheme for Fokker-Planck equations}\label{sub:pre-sde-jko}

Consider the diffusion process \eqref{eq:diffusion-sde-2} starting from  $P \in \calP_2$.
It is known that under generic conditions, as $t\to \infty$, $\rho_t$ converges to the equilibrium distribution of \eqref{eq:diffusion-sde-2} which has density 
\begin{equation}
q \propto e^{-V},    
\end{equation}
 and the convergence is exponentially fast \cite{bolley2012convergence}. 
 The function $V$ is called the  {\it potential function} of $q$.

The evolution of $\rho_t$ by FPE of the diffusion process
can be interpreted as a continuous-time gradient flow under the $\W$-metric in the probability space $\calP_2$. 
The JKO scheme \cite{jordan1998variational} computes a Wasserstein proximal GD which is a time discretization of the gradient flow. 
Specifically, define $G: \calP_2^r \to \R$ as the KL divergence w.r.t. $q$, i.e.,
\begin{equation}\label{eq:def-KL-G}
  \begin{split}
G(\rho) 
&    = {\rm KL} (\rho || q)
    = \calH (\rho) + \calE (\rho),  \\
\calH(\rho)
& =   \int \rho \log \rho,
\quad 
\calE (\rho) = c + \int V \rho,
\end{split}  
\end{equation}
where $c$ is a constant. 
More general $G$ can be considered, see Section \ref{subsec:general-conditions}, and in this work we mainly focus on the case where $G$ is the KL divergence as being considered in \cite{jordan1998variational}.

Under certain regularity condition of $V$, the JKO scheme computes a sequence of distributions $\rho_n$, $n=0,1,...$, starting from $\rho_0\in \calP_2$. 
For a fixed  step  size $\gamma > 0$, and the scheme at the $n$-th step can be written as
\begin{align}\label{eq:JKO-obj-1}
    \rho_{n+1} = \text{arg} \min_{\rho\in \calP_2 }  
    G(\rho) + \fc{1}{2 \gamma} \W^2(\rho_n, \rho) .
\end{align}
The scheme computes the $\W$-proximal Gradient Descent (GD) of $G$ with step size $\gamma$, and can be written as
\begin{align}
    \rho_{n+1} = 
    \text{Prox}_{ \gamma  G}( \rho_n). 
\end{align}

The original JKO paper \cite{jordan1998variational} proved the convergence of the discrete-time solution $\{ \rho_n \}$ (after interpolation over time) to the continuous-time solution $\rho_t$ of the FPE \eqref{eq:FPE} when step size $\gamma \to 0+$. 
In the context of flow-based generative models by neural networks, the discrete-time JKO scheme with finite $\gamma$ was adopted and implemented as a flow network in \cite{xu2022jko}.
Our analysis in this work will prove the exponential convergence of $\rho_n$ to $q$ by the JKO scheme (including learning error), echoing the exponential convergence of the continuous-time dynamic (the FPE). 
This result leads to the guarantee of generating data distributions up to (TV) error $O(\varepsilon)$ in $O(\log (1/\varepsilon))$ JKO steps. 
We will summarize the flow model and introduce theoretical assumptions in Section \ref{sec:setup}.

\section{Setup of JKO flow model and assumptions}\label{sec:setup}

In this section, we summarize the mathematical setup for the JKO flow model and introduce the necessary theoretical assumptions for our analysis.
The guarantee of generating the data distribution will be derived in Section \ref{sec:theory-reverse}
based on the exponential convergence of the $\W$-proximal GD (JKO scheme) in Section \ref{sec:theory-foward}.

\subsection{Forward and reverse processes of JKO flow model}

As has been introduced in Section \ref{subsec:intro-flow},
the flow model implements an ODE model (transport equation), where both the forward process and the reverse process are computed by an invertible Residual Network \cite{iResnet} or a neural-ODE network \cite{chen2018neural,grathwohl2018ffjord}.
The forward process consists of $N$ steps, where each step is computed by a Residual Block --- in the neural-ODE model, this is the neural ODE integration on a sub-time-interval $[t_{n}, t_{n+1}]$, and we also call it a Residual Block. 
The backward process consists of the $N$ steps of the same flow network ``backward in time,'' where each step computes the inverse map of the Residual Block, 
and in the neural-ODE model, this is via integrating the ODE in reverse time. 

The forward and reverse processes (without inversion error)
 are induced by a sequence of transport maps, $T_n$, $n = 1, \ldots, {N}$, which we will define more formally later.
The two processes are  summarized in \eqref{eq:fwd-bwd-process},
\begin{equation}\label{eq:fwd-bwd-process} 
\begin{split}
\text{(forward)} \quad 
& 
p = p_0  
\xrightarrow{T_1}{p_1}  
\xrightarrow{T_2}{} 
\cdots 
\xrightarrow{T_{N}}{p_N} 
\approx q,  \\ 
\text{(reverse)} \quad 
&  p \approx
q_0  \xleftarrow{T_1^{-1}}{ q_1}
\xleftarrow{T_2^{-1}}{ }
\cdots 
\xleftarrow{T_{N}^{-1}}{ q_N}
= q. 
\end{split}
\end{equation}
where $p$ is the density of data distribution (when exists, otherwise a smoothified density by a short time diffusion), and $q$ is the equilibrium density, typically chosen as Gaussian.
Inversion error in the reverse process is considered in Section \ref{subsec:rev-process-inv-err}.

\paragraph{Forward process.}
In the forward process, the algorithm learns a sequence of $T_n$, which transports from data distribution $P$ to the equilibrium distribution $Q$, which is typically the normal distribution, $\calN(0,I)$. 
We denote by $q$ the density of $Q$, and $p$ the density of the data distribution $P$ when there is one.

Following the neural-ODE framework used in \cite{xu2022jko}, each step computes a transport map $T_{n+1}: \R^d \to \R^d$ which is the solution map of the ODE from time $t_n$ to $t_{n+1}$, i.e., 
\begin{equation}\label{eq:def-ode-tn}
T_{n+1}( x_n ) = x(t_{n+1}),
\end{equation}
where $x(t)$ solves $\dot x(t) = \hat{v}( x(t), t )$ on $[t_{n},t_{n+1}]$, $x(t_{n}) = x_n$,
and $\hat v(x,t)$ is the velocity field on $\R^d$ parametrized by the $n$-th Residual Block. 
Equivalently, we have
\begin{equation}\label{eq:def-Tn-2}
    T_{n+1}(x_n) = x_n + \int_{t_n}^{t_{n+1}} \hat v( x(t), t) dt, \quad x(t_n) = x_n.
\end{equation}
In the implementation of the JKO scheme in a flow network, the learning of the $N$ Residual Blocks is conducted progressively for $n=1,\cdots, N$ by minimizing a training objective per step \cite{xu2022jko}. We emphasize that, 
unlike other normalizing flow models, which are trained end-to-end, 
the training procedure here is done step-wise and progressively over the $N$ Residual Blocks. 

Once $T_{n+1}$ is learned, it pushes from $p_n$ to $p_{n+1}$, i.e.,
\begin{equation}\label{eq:pn-pn+1}
    p_{n+1} = (T_{n+1})_\# p_n.
\end{equation}
In our problem, we want the distributions in the intermediate steps to have a density.
To guarantee that this is the case for $p_{n+1}$,
we technically need $T_{n+1}$ to be {\it non-degenerate}. Intuitively, a non-degenerate map cannot collapse a set of finite (Lebesgue) measures into a set of measure zero.
\begin{definition}[Non-degenerate map]\label{eq:def-ND}
Denote by Leb the Lebesgue measure.
$T: \R^d \to \R^d$ is non-degenerate if for any set $A \subset \R^d$ s.t. $Leb (A ) =0$,
then $Leb(T^{-1} (A) ) = 0$.
\end{definition}
If a transport map is non-degenerate, then it pushes forward a density to a distribution that also has density, as shown in the following lemma proved in Appendix \ref{app:proofs-sec-3}.
\begin{lemma}\label{lemma:NG-density}
Suppose $T: \R^d \to \R^d$ is non-degenerate, $P \ll Leb$, then $T_\# P \ll Leb$.
\end{lemma}

Assuming  $T_n$ are all non-degenerate, then the sequence of $p_n$ produced by \eqref{eq:pn-pn+1} all have densities,
starting from $p_0 = p$ the data density.
When data distribution has no density,
we will introduce an initial short-time diffusion that mollifies the data distribution into $\rho_\delta$ which we set to be $p_0$ (see more in Section \ref{subsubsec:P-P2-short-diffusion}). The learning aims that after $N$ steps, the final $p_N$ is close to the equilibrium density $q$.

\paragraph{Reverse process (without inversion error).}
The reverse process computes the inverse of the $N$-steps transport by inverting each $T_n$ in the forward process. We first assume that $T_n$ can be exactly inverted in computation, which allows for a simplified analysis. 
In practice, $T_n^{-1}$ can be implemented by fixed-point iteration \cite{iResnet} or reverse-time ODE integration \cite{grathwohl2018ffjord}. The case when the inverse cannot be exactly computed is discussed in Section \ref{subsec:rev-process-inv-err}, where we need additional assumptions on the closeness of the computed inverse to the true inverse of $T_n$ for our analysis.

The reverse process outputs generated samples, which are aimed to be close in distribution to the data samples, by drawing samples from $q$ and pushing them through the reverse $N$ steps. 
In terms of the sequence of probability densities generated by the process, the reverse process computes
\begin{equation}
    q_{n} = (T_{n+1}^{-1})_\# q_{n+1},
\end{equation}
starting from $q_N = q$ and the output density is $q_0$.
Theoretically, the data processing inequality for the KL divergence applied to invertible transforms (Lemma \ref{lemma:sym-KL}) 
guarantees that if $p_N$ is close to $q_N = q$, then $q_0$ is close to $p_0$, which is the data density (possibly after short-time smoothing). 
This allows us to prove the guarantee of $q_0 \approx p_0$ once we can prove that of $p_N \approx q$, the latter following the convergence of the $\W$-proximal GD up to a hopefully small learning error to be detailed below.

\subsection{Learning assumptions of the forward process}\label{subsec:fwd-process-learning-assumption}

We consider the sequence of densities $p_n$ in the forward process in \eqref{eq:fwd-bwd-process}.
Recall from Section \ref{sub:pre-sde-jko} that for fixed step-size $\gamma > 0$, the $n$-th step classical JKO scheme finds $p_{n+1}$ by minimizing 
\begin{align}\label{eq:JKO-obj-1b}
 \min_{\rho \in \calP_2 } 
 F_{n+1} ( \rho) := 
 G(\rho) + \fc{1}{2 \gamma} \W^2( p_n, \rho),
\end{align}
where $G(\rho) = {\rm KL}( \rho || q)$.
The learning in the $n$-th Residual Block in a JKO flow network computes the minimization via parameterizing the transport $T_{n+1}$. 
Here we briefly review the rational of solving \eqref{eq:JKO-obj-1b} by solving for $T_{n+1}$,
which leads to our assumption of the learned forward process.

\paragraph{JKO step by learning the transport.}
In the right hand side of \eqref{eq:JKO-obj-1b}, when both $p_n$ and $\rho$ are in $\calP_2^r$, 
the Brenier Theorem (Theorem \ref{thm:brenier})
implies the existence of a unique OT map $T$ from $p_n$ to $\rho$.
Consider the following minimization over the transport $T$,
\begin{align}\label{eq:JKO-obj-1c}
   \min_{T : \R^d \to \R^d} 
    G( T_\# p_n  ) 
    + \fc{1}{2 \gamma } \E_{x\sim p_n} \ve{x-T(x)}^2.
\end{align}
The following lemma, proved in Appendix \ref{app:proofs-sec-3},
shows that the minimizer $T$ makes $T_\# p_n \in \calP_2^r$:
\begin{lemma}\label{lemma:pn+1inP2r-exact-T-min}
Suppose $p_n \in \calP_2^r$ makes $G(p_n) < \infty$, 
and $T$ is a minimizer of \eqref{eq:JKO-obj-1c}, 
then $T_\# p_n \in \calP_2^r$.
\end{lemma}
Thus, in \eqref{eq:JKO-obj-1c} it is equivalent to minimize over $T$ that renders $T_\# p_n \in \calP_2^r$ and this means that 
the minimizer $T$ is the OT map. 
One can also verify that \eqref{eq:JKO-obj-1c} is equivalent to \eqref{eq:JKO-obj-1b} in the sense that 
a minimizer $T^*$ of \eqref{eq:JKO-obj-1c} makes $(T^*)_\# p_n$ a minimizer of \eqref{eq:JKO-obj-1b},
and for a minimizer $\rho^*$ of \eqref{eq:JKO-obj-1b} the OT map from $p_n$ to $\rho^*$ is a minimizer of \eqref{eq:JKO-obj-1c}
\cite[Lemma A.1]{xu2022jko}.

\paragraph{Learning error in JKO flow network.}
In the $n$-step of the JKO flow network,  
the transport $T$ in \eqref{eq:JKO-obj-1c}
  is parameterized by a Residual block, and the learning cannot find the $T$ that exactly minimizes  \eqref{eq:JKO-obj-1c}
  (and equivalently \eqref{eq:JKO-obj-1b}) for three reasons: 

\begin{itemize}
\item[(i)] Approximation error: 
The minimization of \eqref{eq:JKO-obj-1c} is over $T$ constrained inside some neural network family $\calT_\Theta$. 
When the function family $\calT_\Theta$ is large enough to express the desired optimal transport from $p_n$ to ${\rm Prox}_{\gamma G} (p_n)$, the solution can approximate the exact minimizer of \eqref{eq:JKO-obj-1b}, but this usually cannot be guaranteed.

\item[(ii)] Finite-sample effect:
The training is computed on empirical data samples, while in this analysis, we focus on the minimization of population loss.

\item[(iii)] Imperfect optimization:
The learning of neural networks is a non-convex optimization typically implemented by Stochastic Gradient Descent (SGD) over mini-batches and there is no guarantee of achieving a minimizer of the empirical loss.
\end{itemize}

As a result,  the learned transport $T_{n+1}$ finds a 
$p_{n+1} = (T_{n+1})_\# p_n$ 
that at most approximately minimizes  \eqref{eq:JKO-obj-1b}. 
While the learned $T_{n+1}$ is usually not the exact minimizer, 
we assume that it is regular enough such that $p_{n+1}$ is still in $\calP_2^r$.
This would hold if $T_{n+1}$ is non-degenerate and also in $L^2(p_n)$ by the following lemma proved in Appendix \ref{app:proofs-sec-3}.

\begin{lemma}\label{lemma:Tp-in-P2r-new}
Let $p \in \calP_2^r$,
then $\{T: \R^d \to \R^d, \, T \in L^2(p) \} = \{T: \R^d \to \R^d, \, T_\# p  \in \calP_2 \} $.
As a result, if $T \in L^2(p)$ and is non-degenerate, then $T_\# p \in \calP_2^r$.
\end{lemma}

When $p_n$ and $p_{n+1}$ are both in $\calP_2^r$, we have a unique invertible OT map from $p_n$ to $p_{n+1}$ by Brenier Theorem, and its (a.e.) inverse is the OT map from $p_{n+1}$ to $p_n$. Specifically, we define
\begin{equation}\label{eq:def-Tn-n+1}
\begin{split}
    & T_n^{n+1} \text{ is the OT map from $p_n$ to $p_{n+1}$, ~~~ $p_n$-a.e.}, \\
    & T_{n+1}^{n} \text{ is the OT map from $p_{n+1}$ to $p_{n}$, ~~~ $p_{n+1}$-a.e.},
\end{split}
\end{equation}
and we have $T_{n+1}^{n} \circ T_n^{n+1} = {\rm I_d}$ $p_n$-a.e., 
$T_{n}^{n+1} \circ T_{n+1}^{n} = {\rm I_d}$ $p_{n+1}$-a.e.
Note that $T_n^{n+1}$ differs from the learned map $T_{n+1}$, see Remark \ref{rk:Tn-vs-OT}.

\paragraph{Assumption on approximate first-order condition.}
For our analysis, we theoretically characterize the error in learning $T_{n+1}$ by quantifying the error in the first-order condition.
Specifically, the $\W$-gradient (strictly speaking, a sub-differential) of $F_{n+1}$ at $\rho$ can be identified as 
\begin{equation}\label{eq:expression-W2-grad-Fn+1}
\nabla_{\W} F_{n+1}( \rho )
= \nabla_{\W} G( \rho )- \frac{ T_{\rho}^n - {\rm I_d} }{\gamma},  
\quad \rho\text{-a.e.}
\end{equation}
where $T_{\rho}^n$ is the OT map from $\rho$ to $p_n$. 
Here we use  $\nabla_{\W} \phi$ to denote the sub-differential $\partial_{\W} \phi$ assuming unique existence to simplify exhibition. 
The formal statement in terms of subdifferential is provided in Lemma \ref{lemma:ot_F} 
(which follows the argument of \cite[Lemma 10.1.2]{ambrosio2005gradient}) 
for more general $G$ (which includes KL divergence as a special case).
For KL divergence $G$, if $V$ is differentiable, the sub-differential $\partial_{\W} G$ is reduced to the unique $\W$-gradient written as
\begin{equation}
\nabla_{\W} G(\rho) = \nabla V + \nabla \log \rho,    
\end{equation}
when $ \rho \in \calP_2^r$ has a well-defined score function.

If $\rho$ is the exact minimizer of \eqref{eq:JKO-obj-1b}, we will have $\nabla_{\W} F_{n+1}( \rho ) = 0$
(and for sub-differential the condition is $0 \in \partial_{\W} F_{n+1}( \rho )$).
At $p_{n+1}$ which is pushed-forward by the learned $T_{n+1}$, we denote the $\W$-gradient of $F_{n+1}$ by the following (recall the definition of $T_{n+1}^n$ as in \eqref{eq:def-Tn-n+1})
\begin{align}
\xi_{n+1}
& := \nabla_{\W} F_{n+1}( p_{n+1} ) \nonumber \\
& = \nabla_{\W} G( p_{n+1} )- \frac{ T_{n+1}^n - {\rm I_d}}{\gamma},
\quad p_{n+1}\text{-a.e.}
\label{eq:def-xi-n+1}
\end{align}
and it is interpreted as sub-differential when needed.
Making an analogy to the (approximate) first-order condition in vector-space optimization, 
we assume that $p_{n+1}$ is close to the exact minimizer such that the (sub-)gradient $\xi_{n+1}$ is small.
In practice, the SGD algorithm to minimize the training objective \eqref{eq:JKO-obj-1c}
(assuming $\calT_\Theta$ is expressive enough to approximate the exact minimizer $T^*$) 
would stop making progress
when the $\W$-gradient vector field $\xi_{n+1}$ evaluated on data samples collectively give a small magnitude.
We characterize this by a small $L^2(p_{n+1})$ norm of the (sub-)gradient $ \xi_{n+1} $ in our theoretical assumption.
The assumptions on the learned transport $T_{n+1}$ are summarized as follows:

\begin{assumption}[Approximate $n$-th step solution]\label{assump:1st-order-condition-error}
The learned transport $T_{n+1}$ is non-degenerate and in $L^2(p_n)$;
it is invertible on $\R^d$ and  $T_{n+1}^{-1}$ is also non-degenerate.
In addition, for some $\varepsilon>0$, 
$\exists \xi_{n+1} \in \partial_{\W} F_{n+1}( p_{n+1} )$ s.t.
\begin{equation}\label{eq:A1-small-|xi_n|}
 \|\xi_{n+1} \|_{p_{n+1}} \le \varepsilon.  
\end{equation}
\end{assumption}

We experimentally verify the smallness of $ \|\xi_{n+1} \|_{p_{n+1}} $ in neural network training of one JKO step. Figure \ref{fig:wass-2d-exp} shows the decrease of the squared $L^2$ norm $ \|\xi_{n+1} \|_{p_{n+1}}^2 $ over training iterations. 
The Wasserstein gradient $ \xi_{n+1} $ is plotted as a vector field in Figure \ref{fig:vec_field-exp}, whose magnitude gradually decreases over training.
See Appendix  \ref{app:exp} for more details.

The error magnitude $\varepsilon$ can be viewed as an algorithmic parameter that  reflects the  accuracy of first-order methods,
and similar assumptions have been made in the analysis of stochastic (noisy) gradient descent in vector space, see, e.g.,  \cite{NemirovskiYudin83, nemirovski2009robust}. 
We emphasize that theoretically, $\varepsilon$ does not need to be small but will enter the final error bound.
For  $T_{n+1}$ satisfying Assumption \ref{assump:1st-order-condition-error},
from $p_n \in \calP_2^r$, $p_{n+1} = (T_{n+1})_\# p_n$  is also in $\calP_2^r$ 
by Lemma \ref{lemma:Tp-in-P2r-new}.
Then the subdifferential $\partial_{\W} F_{n+1}$ can be defined at $p_{n+1}$ and characterized by Lemma \ref{lemma:ot_F}.

\begin{remark}[Assumptions on $T_n$]
The $L^2$ integrability condition of $T_n$ is natural and 
together with non-degeneracy ensures that the next $p_{n}$ is in $\calP_2^r$.
When $T_n^{-1}$ is also non-degenerate, 
the $p_n$'s in the forward process
 and $q_n$'s in the reverse process in \eqref{eq:fwd-bwd-process} all have densities, and this allows us to apply the data processing inequality in both directions (Lemma \ref{lemma:sym-KL}).
 When there is an inversion error in the reverse process,
 we will further impose Lipschitzness of $T_{n}^{-1}$ on $\R^d$ 
 and the Lipschitz constant will theoretically enter the error bound,
see more in Section \ref{subsec:theory-with-inv}.
\end{remark}

\begin{remark}[$T_{n+1}$ and $T_n^{n+1}$]\label{rk:Tn-vs-OT} 
Recall that $T_{n+1}$ is the learned transport map and $T_{n}^{n+1}$ is the OT map.
In our setting (of imperfect minimization in the $n$-th step), 
both $T_{n+1}$ and $T_{n}^{n+1}$ push $p_n$ to $p_{n+1}$ but they are not necessarily the same.
The notion of $T_{n}^{n+1}$ is introduced only for theoretical purposes (the existence and invertibility are by Brenier Theorem), and our theory do not need $T_{n+1}$ to equal $T_{n}^{n+1}$.
On the other hand, one would expect $T_{n+1}$ to approximate $T_{n}^{n+1}$ when the JKO-step optimization \eqref{eq:JKO-obj-1b} is approximately solved, leading to a small $\varepsilon$ in \eqref{eq:A1-small-|xi_n|}.
\end{remark}

\subsection{Reverse process with inversion error}\label{subsec:rev-process-inv-err}

Considering potential inversion error in the reverse process, we denote the sequence of transports as $S_n$ and 
the transported densities as  $\tilde q_n$, that is, 
\begin{equation}
    \tilde q_{n} = (S_{n+1})_\# \tilde  q_{n+1},
\end{equation}
from $\tilde q_N =  q_N= q$. 
The reverse process with and without inversion error is summarized as
\begin{equation}\label{eq:bwd-process-error}
\begin{aligned}
\text{(exact reverse)} 
& &  q_0  \xleftarrow{T_1^{-1}}{ q_1}
\xleftarrow{T_2^{-1}}{ }
\cdots 
\xleftarrow{T_{N}^{-1}}{ q_N} = q, \\
\text{(computed reverse)} 
& &  \tilde q_0  \xleftarrow{S_1}{\tilde q_1}
\xleftarrow{S_2}{ }
\cdots 
\xleftarrow{S_{N}}{ \tilde q_N} = q.
\end{aligned} 
\end{equation}
The computed transport $S_n$ is not the same as $T_n^{-1}$ but the algorithm aims to make the inversion error small. For our theoretical analysis, we make the following assumption on the error. 

\begin{assumption}[Inversion error]
\label{assump:inv-error}
For $n=N, \cdots, 1$, the computed reverse transport $S_n$ 
is non-degenerate, in $ L^2( \tilde q_n)$, and satisfies that
\begin{equation}\label{eq:assump-inv-error}
\| T_n \circ S_n- {\rm I_d} \|_{\tilde q_{n}} \le \varepsilon_{\rm inv}.
\end{equation}
\end{assumption}

In practice, the quantity $\| T_n \circ S_n- {\rm I_d} \|_{\tilde q_{n}}^2 $ can be empirically estimated by sample average, namely the mean-squared error
\begin{align*}
{\rm MSE}_{\rm inv} 
& = \frac{1}{n_{\rm inv}} \sum_{i} \| T_n \circ S_n( x_i ) - x_i \|^2, \\
 x_i & =  S_{n+1} \circ \cdots \circ S_{N-1}  (z_i), \quad z_i \sim q, 
\end{align*}
computed from $n_{\rm inv}$ test samples. 
With sufficiently large $n_{\rm inv}$, one can use ${\rm MSE}_{\rm inv} $ to monitor the inversion error of the reverse process and enhance numerical accuracy when needed. It was empirically shown in \cite{xu2022jko} that the inversion error computed on testing samples (though all the $N$ blocks) can be made small towards the floating-point precision in the neural-ODE model. 

The inversion error objective \eqref{eq:assump-inv-error} resembles the ``cycle consistency'' loss 
in Cycle-GAN \cite{zhu2017unpaired}.
Our theory in Section \ref{subsec:theory-with-inv} suggests that 
keeping the inversion error small is crucial for the success of generating a close-to-data distribution $\tilde q_0$ in the reverse process.
Thus, there may be benefits by introducing the objective \eqref{eq:assump-inv-error} as a regularization in the training of flow-based models, similar to the cycle consistency mechanisms in the GAN literature.

\section{Convergence of forward process}\label{sec:theory-foward}

The current paper mainly concerns the application to flow generative networks where $G$ is the KL divergence. 
In this section, we prove the exponentially fast convergence of the forward process, which applies to potentially a more general class of $G$ as long as 
the subdifferential calculus in $\calP_2$ can be conducted (see Section 10.1 of \cite{ambrosio2005gradient})
and $G$ is strongly convex along generalized geodesics (a.g.g., see Definition \ref{def:convex-agg}), 
which may be of independent interest.

We will revisit the KL divergence $G$ as a special case and prove the generation guarantee of the reverse process in Section \ref{sec:theory-reverse}. 
All proofs and technical lemmas are provided in Appendix \ref{app:proof-sec-4}.

\subsection{Conditions on $G$ and $V$}\label{subsec:general-conditions}

We introduce the more general condition of $G$ needed by the forward process convergence. 
\begin{assumption}[General condition of $G$]\label{assump:genera-G}
    $G:\calP_2 \to (-\infty, +\infty]$ is lower semi-continuous,
    ${\rm Dom} ( G) \subset \calP_2^r$;
    $G$ is $\lambda$-convex a.g.g. in $\calP_2$.
\end{assumption}

The first part of Assumption \ref{assump:genera-G} ensures that the strong subdifferential $\partial_{\W} G(\rho)$ can be defined, see Definition \ref{def:strong-subdiff}.
The strong convexity of $G$ is used to prove the exponential convergence of the (approximated) $\W$-proximimal Gradient Descent in the forward process. 

Next, we show that the KL divergence $G$ satisfies the general condition under certain general conditions of the potential function $V$ plus its strong convexity. We also introduce an upper bound of $\lambda$ due to a rescaling argument.

\begin{figure*}[t]
\begin{minipage}{0.35\textwidth}
\centering
\includegraphics[height=.4\linewidth]{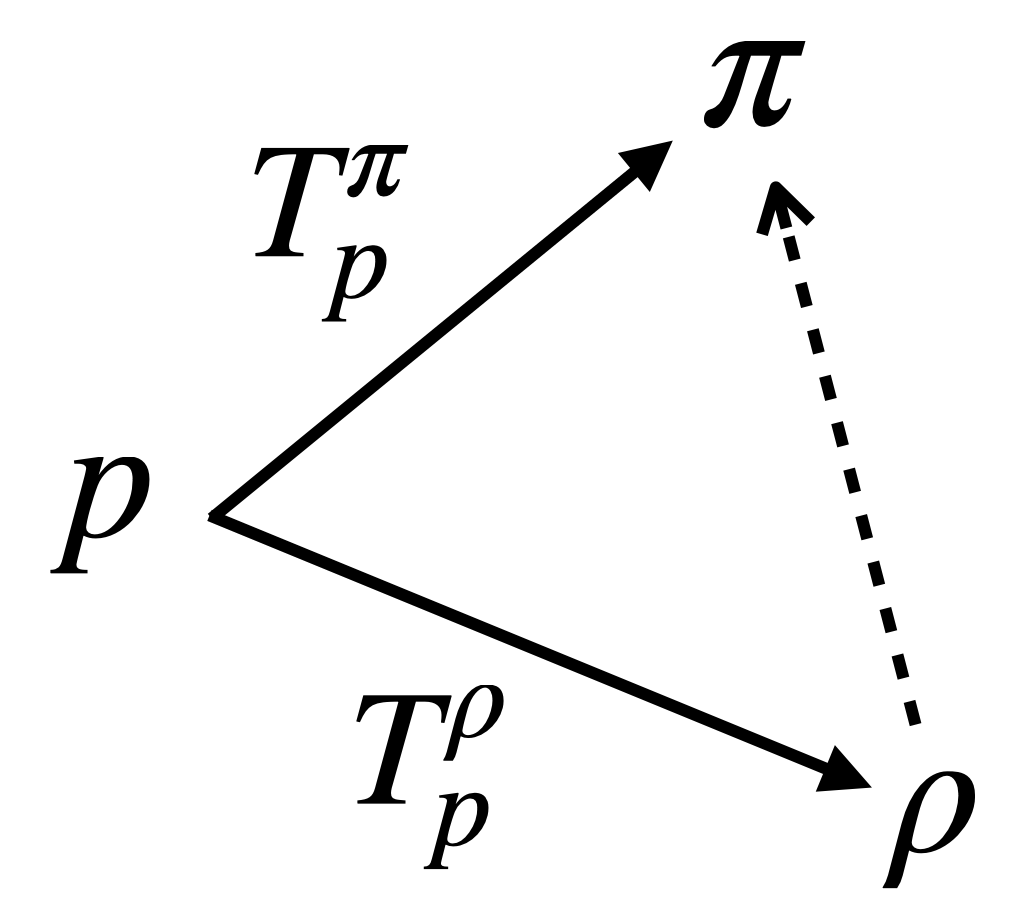} 
\end{minipage}
\hspace{-10pt}
\begin{minipage}{0.6\textwidth}
\renewcommand{\arraystretch}{1.5}
\begin{tabular}{ c }
   $p$, $\pi$, $\rho \in \calP_2^r$, $G$ is $\lambda$-convex a.g.g.,   \\
$ G(\pi) - G(\rho) 
    \geq  \an{ \nabla_{\W} G(\rho) \circ {T_p^\rho}, T_p^\pi - T_p^\rho}_p + \frac{\lambda}{2}\W^2(\pi, \rho) $  \vspace{5pt} \\
    \hline
     $\pi$, $\rho \in \R^d$, $g$ is $\lambda$-convex, \\
    $g(\pi) - g(\rho) 
    \geq  \langle \nabla g(\rho), \pi - \rho \rangle + \frac{\lambda}{2} \| \pi -  \rho \|^2 $
\end{tabular}
\end{minipage}
\caption{
The monotonicity of a.g.g.-convex $G$ in $\calP_2$ proved in Lemma \ref{lemma:mono-G}, as an analog to strong convexity in vector space. We remark that in the usual vector space, the convexity definition does not involve a third vector, since the inner product is uniform; while in probability space, inner product is defined at tangent space associated with $p$.
The dotted line indicates the general geodesic between $\rho$ and $\pi$, see the definitions in Section \ref{subsec:prelim-cal-P2}.
}
\label{fig:mono}
\end{figure*}

\paragraph{KL divergence and $f$-divergence $G$.}

Recall that the KL divergence $G(\rho) = \calH(\rho) + \calE(\rho)$ as defined in \eqref{eq:def-KL-G}, where $\calE(\rho) = \int V \rho$ involves the potential function $V$ of the equilibrium density. 
We introduce the following assumption on $V$:

\begin{assumption}[Condition of $V$ the potential of $q$]\label{assump:V-45}
The potential function $V: \R^d \to (-\infty, \infty]$ is proper, lower semi-continuous, and $V^-:= \max\{-V, 0 \}$ is bounded;
$V(x)$ is $\lambda$-strongly convex on $\R^d$,
and $q \propto e^{-V}$ is in $\calP_2^r$.
\end{assumption}

The first part of Assumption \ref{assump:V-45} is for the sub-differential calculus of $\calE(\rho) = \int V \rho$  in $\calP_2$.
The $\lambda$-strong convexity is used to make $\calE(\rho)$ (and subsequently $G(\rho)$) $\lambda$-convex a.g.g.
Under such condition of $V$, the KL divergence $G(\rho)$ satisfies Assumption \ref{assump:genera-G}, which is verified in Lemma \ref{lemma:KL-satisfies-assump}. Thus our result applies to the KL divergence as being used in the JKO flow model.
For the important case when $q$ is standard normal, $V(x) = \|x \|^2/2$, and $\lambda = 1$.

More generally, consider the $f$-divergence $G(\rho) =D_f( \rho ||q):= \int f(\rho /q) q$, $f: (0, \infty) \to \R$ being convex, lower-semicontinuous and $f(1)=0$.
KL divergence corresponds to a special case where $f(x) = x\log x$.
Under Assumption \ref{assump:V-45}, we have $q \propto e^{-V}$ and is log-concave, and then $G(\rho)$ is convex a.g.g. in $\calP_2^r$ (under additional technical conditions on $f$) \cite[Theorem 9.4.12]{ambrosio2005gradient}. 
We think it is possible to establish the $\lambda$-convexity a.g.g. of $f$-divergence $G$ for a certain class of $f$, and details are postponed here.

\paragraph{Positive and bounded $\lambda$.}
Note that for strongly convex $V$ we can use a scaling argument to make $\lambda$ bounded to be $O(1)$ without loss of generality. Specifically, for $V$ that is $\lambda$-convex on $\R^d$,  the function $x \mapsto V( ax )$ for $a>0$ is ($a^2\lambda$)-convex. This means that for $q$ that has a strongly convex $V$ as the potential function, one can rescale samples from $q$ to make $V$ strongly convex with $\lambda \le 1$. 
In the case where $G$ is KL divergence, the $\lambda$-convexity of $G$ has the same $\lambda$ as that of $V$.
Thus, for the general $G$ we assume its $\lambda$ is also bounded by 1. 

\begin{assumption}[$ \lambda$ bounded]\label{assump:lambda<=1}
In Assumptions \ref{assump:genera-G} and \ref{assump:V-45}, $ 0 < \lambda \le 1$.
\end{assumption}

Our technique can potentially extend to analyze the $\lambda = 0 $ case, where an algebraic $O(1/n)$ convergence rate is expected instead of the exponential rate proved in Theorem \ref{thm:N-step-forward-no-inv-error}. 
For the application to flow-based generative model, one would need the equilibrium density $q \propto e^{-V}$ convenient to sample from, and thus the normal density (corresponding to $V(x) = \| x\|^2/2$) is the most common choice and the other choices usually render $\lambda > 0$ (to enable fast sampling of the starting distribution). 
We thus leave the $\lambda =0$ case to future work.

\subsection{Evolution Variational Inequality and convergence of the forward process}\label{subsec:convergence-forward}

The a.g.g. $\lambda$-convexity of $G$ leads to the following lemma, which is important for our analysis.
All proofs in this section are provided in Appendix \ref{app:proof-sec-4}.

\begin{lemma}[Monotonicity of $G$]
\label{lemma:mono-G}
 Let $p, \rho \in \calP_2^r$, $\pi \in \calP_2$, 
 and denote by $T_p^\rho$  and $T_p^\pi$
 the OT maps from $p$ to $\rho$ and to $\pi$ respectively. 
 Suppose $G$ satisfies Assumption \ref{assump:genera-G},
 then for any $\eta \in \partial_{\W} G(\rho)$, 
\begin{align*}
    G(\pi) - G(\rho) 
    \geq  \an{\eta \circ {T_p^\rho}, T_p^\pi - T_p^\rho}_p + \frac{\lambda}{2}\W^2(\pi, \rho).
\end{align*}
\end{lemma}

The relationship among $p, \rho, \pi$ is illustrated in Figure~\ref{fig:mono}, which also includes an analog to the strong-convex function in Euclidean space.
This lemma extends Lemma 4 in \cite{salim2020wasserstein}
and originally the argument in Section 10.1.1.B of \cite{ambrosio2005gradient}.
We include a proof in Appendix \ref{app:proof-sec-4} for completeness.

Based on the monotonicity lemma and the condition of small strong subdifferential $\xi_{n+1}$ in Assumption \ref{assump:1st-order-condition-error}, 
we are ready to derive the discrete-time {\it Evolution Variational Inequality} (EVI) \cite[Chapter 4]{ambrosio2005gradient}
for the (approximate) JKO scheme.

\begin{lemma}[EVI for approximate JKO step]
\label{lemma:evi}
Given $\pi \in \calP_2$,
suppose $G$ satisfies Assumption \ref{assump:genera-G} with  $\lambda \in (0,1]$,
and $0 < \gamma < 2$.
If $p_0 \in \calP_2^r$, and Assumption \ref{assump:1st-order-condition-error} holds for $n=0, 1, \cdots $, then for all $n$, 
\begin{align}\label{eq:lemma-evi}
   &\pa{1+\frac{ \gamma \lambda }{2}}\W^2(p_{n+1}, \pi ) 
   +  2\gamma \pa{ G(p_{n+1}) - G(\pi)}\notag\\ 
   & ~~~
   \leq 
   \W^2(p_n,\pi) 
   + \frac{2 \gamma }{\lambda} \varepsilon^2.
\end{align}
\end{lemma}

The condition of  $\lambda \le 1$ and $\gamma < 2$ can be replaced with other constants, and our analysis will give similar results. 
Specifically, the upper bound 1 of $\lambda$ in Assumption \ref{assump:lambda<=1} is a generic choice and we keep it, then replacing the requirement of $\gamma < 2$ with $\gamma < \gamma_{\rm max}$ for another $\gamma_{\rm max} > 1$ will only affect the constant in the final bound and does not affect the order.
For exhibition simplicity, we give our analysis under $\lambda \le 1$ and $\gamma < 2$ without loss of generality.
We also provide some rationale for upper bounding the step size $\gamma$ motivated by practice:
First, it has been empirically observed that successful computation of the JKO flow model in practice needs the step size not to exceed a certain maximum value, which is an algorithmic parameter \cite{xu2022jko}. Setting the step size too large may lead to difficulty in training the Residual blocks as well as in maintaining small inversion errors.
Meanwhile, from the formulation of the JKO scheme, it can be seen that for large $\gamma$, 
the proximal GD in \eqref{eq:JKO-obj-1} approaches the global minimization of $G(\rho)$, which asks for the flow to transport from the current density to the target density $q$ in one step. 
Though the proximal GD (as a backward Euler scheme) does not impose a step-size constraint, the optimization problem \eqref{eq:JKO-obj-1} is, in principle, easier with a small (but no need to converge to zero) step size.

The EVI directly leads to  the $N$-step convergence of the forward process,
which achieves $O(\varepsilon)$  $\W$-error
and  $O(\varepsilon^2)$ gap from the optimal objective value
in $N \lesssim \log (1/\varepsilon)$ JKO steps.

\begin{theorem}[Convergence of forward process]\label{thm:N-step-forward-no-inv-error}
Suppose $q \in \calP_2$ is the global minimum of $G$, 
and the other assumptions are the same as Lemma \ref{lemma:evi}, then for $n = 1, 2, \cdots$,
\begin{align}\label{eq:thm-exp-rate}
    \W^2(p_n, q)  \leq  \left( 1+ \frac{\gamma \lambda}{2} \right)^{-n} \W^2(p_0, q) + \fc{4\varepsilon^2}{\lambda^2}.
\end{align}
In particular, if
\begin{equation}\label{eq:thm-exp-n-threshold}
n \ge  \frac{8 }{  \gamma \lambda }\left( \log  \W(p_0, q)  + \log ({\lambda}/{\varepsilon})  \right),
\end{equation}
then 
\begin{equation}\label{eq:thm-exp-bound}
\W(p_n, q) \le  \sqrt{5} \frac{\varepsilon}{\lambda},
 \quad
G(p_{n+1}) - G(q) \le  \frac{9}{2 \gamma } \left(\frac{\varepsilon}{\lambda} \right)^2.
\end{equation}
\end{theorem}

\begin{remark}[Comparison to \cite{salim2020wasserstein}]\label{remark-forward-backward}
The convergence rates of Wasserstein proximal GD were previously studied in \cite{salim2020wasserstein}, and our proof techniques, namely the monotonicity of $G$ plus discrete-time EVI, are similar to the analysis therein. However, the setups differ in several aspects: first, we consider the ``fully-backward'' proximal GD, i.e., the JKO scheme, while \cite{salim2020wasserstein} focuses on the forward-backward scheme to minimize $G$ having a decomposed form $\calE + \calH$ (which does cover the KL convergence as a special case). 
Second, \cite{salim2020wasserstein} assumed the exact solution of the proximal step while our analysis takes into account the error $\varepsilon$ in the first-order condition Assumption \ref{assump:1st-order-condition-error},
which is more realistic for neural network-based learning. 
At last, \cite{salim2020wasserstein} assumed $L$-smoothness of $V$, namely $\nabla V$ is $L$-Lipschitz, and step size $\gamma < 1/L$, as a result of the forward step in the splitting scheme, which is not needed in our fully backward scheme.
The motivation for \cite{salim2020wasserstein} is to understand the discretized Wasserstein gradient flow and to recover the same convergence rates as the (forward-backward) proximal GD in the vector space.
Our analysis is motivated by the JKO flow network, and the forward process convergence is an intermediate result to prove the generation guarantee of the reverse process. 
\end{remark}

\section{Generation guarantee of reverse process}\label{sec:theory-reverse}

In this section, we first consider the reverse process as in \eqref{eq:fwd-bwd-process}, called the {\it exact} reserves process, when there is no inversion error.
In Section \ref{subsec:bwd-no-inv},  we prove a KL (and TV) guarantee of generating by $q_0$ any data distribution $P$ in $\calP_2^r$, and extend to $P$ with no density by introducing a short-time initial diffusion.

Taking into account the inversion error, we consider the sequence $\tilde q_n$ induced by $S_n$ in \eqref{eq:bwd-process-error}, called the {\it computed} reverse process, where the inversion error satisfies Assumption \ref{assump:inv-error}.
In Section \ref{subsec:theory-with-inv}, we prove a closeness bound of $\tilde q_0$ to $q_0$ in $\W$, which leads to a $\W$-KL mixed generation guarantee of $\tilde q_0$ based on the proved guarantee of $q_0$.

\subsection{Convergence guarantee without inversion error}\label{subsec:bwd-no-inv}

We start by presenting convergence analysis assuming there are no errors in the reverse process; then, we extend to the more practical situation considering inversion errors.
All proofs in this section are given in Appendix \ref{app:proofs-sec5}.

\subsubsection{KL (TV) guarantee of generating $P$ with density}\label{subsubsec:bwd-no-inv-P2r}

Consider the transport map over the $N$ steps in \eqref{eq:fwd-bwd-process} denoted as 
\begin{equation}\label{eq:def-T1N}
T_{1}^N :=  T_{N} \circ \cdots \circ T_1.
\end{equation}
Since each $T_n$ is invertible (Assumption \ref{assump:1st-order-condition-error}), the overall mapping $T_1^N$ is also invertible.
We have 
\[ 
p_N = (T_1^N)_\# p_0,
\quad 
q_N = (T_1^N)_\# q_0.
\]
The following lemma follows from the data processing inequality of KL, 
which allows us to obtain KL bound of $p_0 =p $ and $q_0$ from that of $p_N$ and $q_N = q$.

\begin{lemma}[Bi-direction data processing inequality]\label{lemma:sym-KL}
    If $T: \R^d \to \R^d$ is invertible 
    and for two densities $p$ and $q$  on $\R^d$,
    $T_\# p$ and $T_\# q$ also have densities, then 
    \[
    {\rm KL}(p || q)  = {\rm KL}( T_\# p || T_\# q) . 
    \]
\end{lemma}

The following corollary establishes an $O(\varepsilon^2)$ KL bound 
in $N \lesssim \log (1/\varepsilon)$ JKO steps, which implies an $O(\varepsilon)$ TV bound by Pinsker's inequality.
\begin{corollary}[KL guarantee for $P \in \calP_2^r$]\label{cor:KL-P-P2r}
Suppose $G(\rho) = {\rm KL}(\rho || q)$,
the potential function $V$ satisfies Assumption \ref{assump:V-45} with $\lambda \in (0,1 ]$,
and $0 < \gamma < 2$. 
Suppose $P \in \calP_2^r$ with density $p$, let $p_0 = p$, and Assumption \ref{assump:1st-order-condition-error} holds for some $\varepsilon$ and all $n$.
Then, let 
\begin{equation}\label{eq:def-big-N} 
N  = \left\lceil \frac{8 }{ \gamma \lambda} \left( \log  \W(p_0, q)  + \log ({\lambda}/{\varepsilon})  \right)  \right\rceil ,
\end{equation}
the generated density $q_0$ of the reverse process satisfies that 
\begin{equation}\label{eq:bound-cor-KL-P-Pr2}
{\rm KL}( p || q_0) \le \frac{9}{2\gamma} \left(\frac{\varepsilon }{\lambda}\right)^2,
\quad
{\rm TV}( p, q_0) \le \frac{3}{2 \sqrt{\gamma}} \frac{\varepsilon}{\lambda}.
\end{equation}
\end{corollary}

\begin{remark}[Extension to $f$-divergence]\label{rk:f-div-guarantee}
As has been discussed in Section \ref{subsec:general-conditions}, it is possible to show that 
$G(\rho) = D_f( \rho || q)$ satisfies Assumption \ref{assump:genera-G} (possibly under additional conditions of $f$),
and then the convergence of the forward process, Theorem \ref{thm:N-step-forward-no-inv-error}, extends to such $f$-divergence $G$.
In addition, data processing inequality holds generally for $f$-divergence, and thus Lemma \ref{lemma:sym-KL} also extends.  
As a result, Corollary \ref{cor:KL-P-P2r} can potentially extend to certain $f$-divergences and show a guarantee of $D_f( p || q_0)$ in $N$ JKO steps.
\end{remark}

\subsubsection{Guarantee of generating $P \in \calP_2$ up to initial short diffusion}\label{subsubsec:P-P2-short-diffusion}

For $P \in \calP_2$ that may not have a density, we first obtain $\rho_\delta \in \calP_2^r$ 
that is close to $P$ in $\W$ by a short-time initial diffusion (specifically, the OU process as introduced in Section \ref{sub:pre-sde-jko})
 up to time $\delta >0$, as shown in Lemma \ref{lemma:rho-delta-W2}.
The short-time initial diffusion was used in \cite{lee2023convergence} 
and called ``early stopping'' in \cite{chen2023improved}.
It is also used in practice by flow model \cite{xu2022jko}
as well as score-based diffusion models to bypass the irregularity of data distribution \cite{song2021score}.
In principle, one can also use the Brownian motion only (corresponding to convolving $P$ with Gaussian kernel) to obtain $\rho_\delta$. Here we use the OU process to stay in line with the literature. 

The introduction of $\rho_\delta$ allows us to prove a guarantee of ${\rm KL}( \rho_\delta|| q_0)$ in the following corollary, which is the same type of result as \cite[Theorem 2]{chen2023improved}.
\begin{corollary}[KL guarantee for $P \in \calP_2$ from $\rho_\delta$]\label{cor:KL-mixed-P2}
Suppose $P \in \calP_2$, and the conditions on $G$, $V$, $\lambda$ and $\gamma$ are the same as in Corollary \ref{cor:KL-P-P2r}.
Then $\forall \varepsilon' >0$, there exists $\delta > 0$ s.t.  $\W( P, \rho_\delta) < \varepsilon'$ 
and,
with $p_0 = \rho_\delta$ and Assumption \ref{assump:1st-order-condition-error} holds for some $\varepsilon$ and all $n$,
let $N$ as in \eqref{eq:def-big-N}, 
the generated density $q_0$ of the reverse process 
makes 
${\rm KL}( \rho_\delta || q_0)$
and ${\rm TV}( \rho_\delta, q_0)$
satisfy the same bounds 
as in \eqref{eq:bound-cor-KL-P-Pr2}.
\end{corollary}

The corollary shows that there is a density $\rho_\delta \in \calP_2^r$ that is arbitrarily close to $P$ in $\W$, such that the output density $q_0$ of the reverse process can approximate $\rho_\delta$ up to the same error as in Corollary \ref{cor:KL-P-P2r}.
Note that the corollary holds when the potential function $V$ of $q$ satisfies the general condition Assumption \ref{assump:V-45}. 
The OU process in the proof (Lemma \ref{lemma:rho-delta-W2}) 
is only used in constructing $\rho_\delta$, and there are other means to construct  the surrogate initial density $\rho_\delta$.

\subsection{Convergence guarantee with inversion error}\label{subsec:theory-with-inv}

Recall the set-up from Section \ref{subsec:rev-process-inv-err}.
To prove the $\W$ control between $\tilde q_0$ and $q_0$, we first introduce a Lipschitz condition on the inverse of the learned transport map $T_n$ and explain the motivation.

\subsubsection{Lipschitz constant of computed transport maps}\label{subsec:Tn-Lip}

Previously in Assumption \ref{assump:1st-order-condition-error}, we required that both $T_n$ and its inverse are non-degenerate.
Here, we further require that $T_n^{-1}$ is globally Lipschitz on $\R^d$ with a uniform Lipschitz constant.

\begin{assumption}[Lipschitz condition on $T_n^{-1}$]
\label{assump:Tn-Lip-bound}
There is $K > 0$ s.t.  $T_n^{-1}$ is Lipschitz on $\R^d$ with Lipschitz constant $e^{\gamma K}$ for all $n=N, \ldots, 1$.
\end{assumption}

The assumed Lipschitz constants are theoretical and motivated by neural ODE models, to be detailed below.
Our analysis of $\W( \tilde q_0, q_0 )$ applies to any type of flow network (like invertible ResNet) as long as 
the needed assumptions on $T_n$ and $S_n$ hold.

We justify the assumptions on $T_n$, $T_n^{-1}$ and $S_n$
under the framework of neural ODE flow, namely \eqref{eq:def-ode-tn}\eqref{eq:def-Tn-2}, including the Lipschitz constant $e^{\gamma K}$ of $T_n^{-1}$.
Specifically, by the elementary Lemma \ref{lemma:ode-lip} proved in Appendix \ref{app:proofs-sec5}, we know that if $T_{n+1}$ can be numerically exactly computed as \eqref{eq:def-Tn-2} and $\hat v( x,t)$ on $\R^d \times [t_n, t_{n+1}]$ satisfies a uniform $x$-Lipschitz condition with Lipschitz constant $K$, then both $T_{n+1}$ and its inverse are Lipschitz on $\R^d$ with Lipschitz constant $e^{ \gamma K}$.
We will assume the same $K$ throughout time for simplicity.
In practice,  Lipschitz regularization techniques can be applied to the neural network parametrized $\hat v(x,t)$,
and the global Lipchitz bound of $\hat v$ on $\R^d$ can be achieved by ``clipping''  $\hat v$ to vanish outside some bounded domain of $x$. 
Meanwhile, note that if $T: \R^d \to \R^d$ is invertible and $T^{-1}$ is globally Lipschitz on $\R^d$, then $T$ is non-degenerate.
Thus $T_{n+1}$ and $T_{n+1}^{-1}$ both being non-degenerate are implied by (and weaker than) the global Lipschitzness of $T_{n+1}$ and its inverse.
In addition, in a neural-ODE-based flow model, the reverse process is by integrating the neural ODE in reverse time,
and thus we can expect similar properties of $S_{n}$.

While the computed transport $T_{n}$ and $S_{n}$ often differ from the exact numerical integration of the ODE, we still expect the Lipschitz property to retain. For general flow models, which may not be neural ODE, we impose the same theoretical assumptions. 
At last, the global Lipschitz condition 
may be theoretically relaxed by combining with truncation arguments of the probability distributions, which is postponed here.

\subsubsection{$\W$-control of the computed reverse process from the exact one}

\begin{proposition}\label{prop:w2-q0}
Suppose in \eqref{eq:bwd-process-error}, $q_N = \tilde q_N = q \in \calP_2^r$, and
the computed transport maps $T_n$ and $S_n$ satisfy
Assumptions \ref{assump:1st-order-condition-error}, \ref{assump:inv-error} and \ref{assump:Tn-Lip-bound}.
Then all $q_n$ and $\tilde q_n$ are in $\calP_2^r$ and
\begin{equation}\label{eq:W-tilq0-q0}
\W( \tilde q_0, q_0 ) \le \frac{ \varepsilon_{\rm inv}  }{  \gamma K  } e^{\gamma K (N+1)}.
\end{equation}
\end{proposition}
A continuous-time counterpart of Proposition \ref{prop:w2-q0} was derived in \cite[Proposition 3]{albergo2023building}.
We include a proof in Appendix \ref{app:proofs-sec5} for completeness.
The proof uses a coupling argument of the (discrete-time) ODE flow,
which as has been pointed out in \cite{chen2024probability},
obtains a growing factor $e^{\gamma K N}$ in the $\W$-bound as shown in \eqref{eq:W-tilq0-q0}.
To overcome this exponential factor, \cite{chen2024probability} adopted an SDE corrector step.
Here, without involving any corrector step,
we show that the factor $e^{\gamma K N}$ can be controlled at the order of some negative power of $\varepsilon$ thanks to the exponential convergence in the forward process.
This is because $N$ can be chosen to be at the order of $\log (1/\varepsilon)$ as in \eqref{eq:def-big-N},
then $e^{\gamma K N}$ can be made $O(\varepsilon^{-\alpha})$ for some $\alpha > 0$.
As a result, the $\W$-error \eqref{eq:W-tilq0-q0} can be suppressed if $\varepsilon_{\rm inv}$ can be made smaller than a higher power of $\varepsilon$. 

More specifically, combined with the analysis of the forward process, we arrive at the following guarantee, proved in  Appendix \ref{app:proofs-sec5}.

\begin{corollary}[Mixed bound with inversion error]\label{cor:mixed-bound-inv}
Suppose $G(\rho) = {\rm KL}(\rho || q)$,
the potential function $V$ satisfies Assumption \ref{assump:V-45} with $\lambda \in (0,1 ]$,
and $0 < \gamma < 2$.
Suppose the computed transports maps $T_n$ and $S_n$ satisfy the  Assumptions
\ref{assump:1st-order-condition-error},
\ref{assump:inv-error},
\ref{assump:Tn-Lip-bound} for some $\varepsilon$ and $\varepsilon_{\rm inv}$ for all $n$.
Suppose $P \in \calP_2^r$ with density $p$, let $p_0 = p$ and $N$ as in \eqref{eq:def-big-N}, then
the generated density $\tilde q_0$ of the computed reverse process satisfies that 
\begin{equation}\label{eq:w2-tidleq0-q0-bound-final}
\W( \tilde q_0, q_0) \le \frac{e^{2 \gamma K}}{ \gamma K} (\W(p_0, q) \lambda)^{8K /\lambda} \frac{ \varepsilon_{\rm inv}}{\varepsilon^{ 8 K /\lambda }},
\end{equation}
and $q_0$ satisfies the KL and TV bounds to $p$ as in \eqref{eq:bound-cor-KL-P-Pr2}.
\end{corollary}

\begin{remark}[$O(\varepsilon)$ error and need for small $\varepsilon_{\rm inv}$]
The corollary implies that if $\varepsilon_{\rm inv}$ can be made small, then the $\W$ bound can be made equal to or smaller than $\varepsilon$ in order. 
For example, if $\varepsilon_{\rm inv} = O( \varepsilon^{8K/\lambda +1} )$, then  we have $\W( \tilde q_0, q_0)  = O(\varepsilon)$. This suggests that if one focuses on getting ${\rm KL }(p_N || q) $ small in the forward process, then maintaining an inversion error small is crucial for the generation quality of the flow model in the reverse process.
\end{remark}

At last, when $P$ is merely in $\calP_2$ and does not have density, then one can start the forward process from $p_0 =\rho_\delta$  same as in Section \ref{subsubsec:P-P2-short-diffusion}.
Then we have the same $\W$-bound between $\tilde q_0$ and $q_0$ as in \eqref{eq:w2-tidleq0-q0-bound-final}, 
and $q_0$ is close to $\rho_\delta$ in the sense of Corollary \ref{cor:KL-mixed-P2}.

\section{Discussion}

The work can be extended in several directions. 
First, it is interesting to see if 
the assumption on learning in the forward process, Assumption \ref{assump:1st-order-condition-error}, 
can be derived from further analysis of the neural network learning, e.g., the approximation and optimization error (c.f. the list of sources of errors in Section 3.2). 
The current work does not contain such analysis and instead handles the goodness of the learned $T_n$ by a single assumption.
 In particular, it would be of interest to theoretically justify the assumed ``first order condition,'' i.e., the smallness of the $\W$-gradient $\xi_{n+1}$ in \eqref{eq:A1-small-|xi_n|}, by analyzing the convergence of the optimization.
 One possibility is by 
 showing the weak convergence of the learned $p_{n+1}$ to the exact minimizer $p^*$ of $F_{n+1}$
 and then utilizing the convergence of  $\nabla_{\W} F_{n+1}(p)$ to $\nabla_{\W} F_{n+1}(p^*) = 0$
 in a proper sense \cite[Section 5.4]{ambrosio2005gradient}.
Second, the current generation result only covers the case of $G$ being the KL divergence. An extension to the cases when $G$ is other types of divergence, such as $f$-divergence (see Remark \ref{rk:f-div-guarantee}), will broaden the scope of the result.
Third, our theory uses the population quantities throughout. A finite-sample analysis, which can be based on our population analysis, will provide statistical convergence rates in addition to the current result.

Meanwhile, the JKO scheme computes a fully backward proximal GD. Given the existing convergence rates of the other Wasserstein GD \cite{salim2020wasserstein,kent2021modified}, one would expect that a variety of first-order Wasserstein GD schemes can be applied to progressive flow models and the theoretical guarantees can be derived similarly to the JKO scheme. 
We also note the connection between the JKO scheme and learning of the score function, at least in the limit of small step size \cite[Section 3.2]{xu2022jko}. Given the growing literature on the analysis of score-based diffusion models, it can be worthwhile to investigate this connection further to develop new theories for the ODE flow models. 

Finally, it would be interesting to use theory to guide practice and to develop new or improved methodologies of flow-based generative models. 
As discussed in Section \ref{subsec:rev-process-inv-err}, one may consider incorporating the inversion error as part of training loss to enforce the accuracy of the reverse process. 
The potential theoretical extension to $f$-divergences also suggests utilizing more general $f$-divergence as the per-step training objective in JKO flow networks, e.g., by adopting techniques in $f$-GAN \cite{nowozin2016f}.
It would be interesting to explore different choices of $f$ and the relationships among the $f$-divergences \cite{harremoes2011pairs}.
In addition, our theory indicates that using larger step-size $\gamma$ leads to a shorter sequence of Residual Blocks in the network architecture (as long as the optimization in each JKO step can be efficiently solved). It would be natural to consider an adaptive choice of $\gamma$ in practice and also in extending the theory.

\section*{Acknowledgement}

The authors thank the anonymous reviewers for helpful comments and suggestions. 
We would like to thank Jose Blanchet, Holden Lee, and Adil Salim for their helpful discussions. 
Thanks to Chen Xu for help with the neural network training in the numerical results.
The work of JL and YT is supported in part by the National Science Foundation via grants DMS-2012286 and DMS-2309378. 
The work of XC and YX is supported by NSF DMS-2134037,
and YX is also partially supported by 
an NSF CAREER CCF-1650913, 
CMMI-2015787, CMMI-2112533, DMS-1938106, DMS-1830210, and the Coca-Cola Foundation;
XC is also partially supported by 
NSF DMS-2237842
and Simons Foundation.

\bibliographystyle{plain}
\bibliography{flow}



\appendix

\setcounter{figure}{0} 
\renewcommand{\thefigure}{A.\arabic{figure}}
\setcounter{table}{0} 
\renewcommand{\thetable}{A.\arabic{table}}
\setcounter{equation}{0} 
\renewcommand{\theequation}{A.\arabic{equation}}
\setcounter{remark}{0} 
\renewcommand{\theremark}{A.\arabic{remark}}
\renewcommand{\thelemma}{A.\arabic{lemma}}

\section{Proofs and lemmas in Section \ref{sec:setup}}\label{app:proofs-sec-3}

\subsection{Lemma on the $\W$-(sub)gradient}

\begin{lemma}\label{lemma:ot_F}
Suppose $G:\calP_2(X)\rightarrow (-\infty, +\infty]$ is lower semi-continuous  and  ${\rm Dom}( G) \subset \calP_2^r$.
Let $\gamma > 0$, $p \in \calP_2^r$, and 
\begin{align}
 F ( \rho) = 
 G(\rho) + \fc{1}{2 \gamma} \W^2( p, \rho).
\end{align}
If (at $\rho \in \calP_2^r$, $\partial_{\W}F(\rho)$ is non empty and) $\xi \in \partial_{\W} F (\rho)$, then 
\[
\xi + \frac{T_\rho^p - {\rm I_d } }{\gamma} \in \partial_{\W} G (\rho).
\]
\end{lemma}

The argument follows that in Lemma 10.1.2 in \cite{ambrosio2005gradient}, and we include a proof for completeness.

\begin{proof}[Proof of Lemma \ref{lemma:ot_F}]
We are to verify that 
\[
\eta := \xi + \frac{T_\rho^p - {\rm I_d } }{\gamma}  
\]
is a strong subdifferential of $G$ at $\rho$. By Definition \ref{def:strong-subdiff}, 
it suffices to show that  for any $v \in L^2(\rho)$ and $\delta \to 0$, 
\begin{equation}\label{eq:pf-Fdiff-1}
	G ( ( {\rm I_d} + \delta v)_\# \rho ) - G( \rho)  
	\ge  \delta \langle \eta, v \rangle_\rho + o(\delta).
\end{equation}

By construction, 
\[
 \langle \eta, v \rangle_\rho
 =   \langle \xi   , v \rangle_\rho +  \frac{1}{\gamma}\langle {T_\rho^p - {\rm I_d } }, v \rangle_\rho,
\]
and since $\xi$ is a strong subdifferential of $F$  at $\rho$, 
\[
	F( ( {\rm I_d} + \delta v)_\# \rho ) - F( \rho)  
	\ge  \delta \langle \xi, v \rangle_\rho + o(\delta).
\]
Combining the two and by the definition of $F$, we can deduce \eqref{eq:pf-Fdiff-1} as long as we can show that 
\begin{equation}\label{eq:pf-Fdiff-2}
\frac{1}{2} \W(p, ( {\rm I_d} + \delta v)_\# \rho)^2 + o(\delta)
\le \frac{1}{2} \W(p, \rho)^2 -  \langle T_\rho^p - {\rm I_d}, \delta v \rangle_\rho .
\end{equation}

To show \eqref{eq:pf-Fdiff-2}, note that by Brenier Theorem (ii),
\[
\W(p, \rho)^2 = \int \| x - T_\rho^p(x) \|^2 \rho(x) dx = \| {\rm I_d}  - T_\rho^p \|_\rho^2.
\]
Thus,
\begin{align}
&\quad\  \frac{1}{2} \W(p, \rho)^2  -  \langle T_\rho^p - {\rm I_d}, \delta v \rangle_\rho \nonumber\\
& = \frac{1}{2} \| {\rm I_d}  - T_\rho^p \|_\rho^2 + \langle {\rm I_d} - T_\rho^p, \delta v \rangle_\rho  \nonumber \\
& =\frac{1}{2} \| ({\rm I_d}  + \delta v) - T_\rho^p  \|_\rho^2 - \frac{1}{2} \| \delta v \|_\rho^2. \label{eq:pf-Fidff-3}
\end{align}
Note that, because $v \in L^2(\rho)$,
\[
\| \delta v \|_\rho^2 = O(\delta^2)
\]
and 
\begin{align*}
   \|  ({\rm I_d}  + \delta v) - T_\rho^p  \|_\rho^2
& = \int_{\R^d}  \| ({\rm I_d}  + \delta v)(x) - T_\rho^p (x) \|^2 \rho(x) dx
\\ &\ge  \W (  ({\rm I_d}  + \delta v)_\# \rho , p)^2.
\end{align*}
Putting together, this gives that the r.h.s. of \eqref{eq:pf-Fidff-3} is greater than or equal to 
\[
\frac{1}{2}  \W (  ({\rm I_d}  + \delta v)_\# \rho , p)^2 + O(\delta^2) 
\]
which implies \eqref{eq:pf-Fdiff-2}.
\end{proof}

\subsection{Proofs of Lemmas}

\begin{proof}[Proof of Lemma \ref{lemma:NG-density}]
It suffices to show that for any $A$ s.t. $Leb(A) = 0$, $T_\# P (A) = 0$.
By the definition of push-forward, $T_\# P (A) = P( T^{-1} (A) )  $,
which is zero because $Leb( T^{-1} (A)  ) = 0$ (since $T$ is  non-degenerate)
and  $P \ll Leb$ .
\end{proof}
\begin{proof}[Proof of Lemma \ref{lemma:pn+1inP2r-exact-T-min}]
First, the minimizer makes the r.h.s. finite because $T = {\rm I_d}$ makes it finite:
When $T$ is identity, the r.h.s. equals $G(p_n) < \infty$.
As a result, $\tilde p:= T_\# p_n$ needs to have density because otherwise the KL divergence 
$G(\tilde p ) = +\infty$.

It remains to show that $M_2( \tilde p ) < \infty$. By definition,
\begin{align*}
M_2( \tilde p) 
& = \int_{\R^d} \| x \|^2 \tilde p(x) dx  \\
& = \E_{x \sim p_n} \| T (x) \|^2 \\
& \le 2 ( \E_{x \sim p_n} \| x \|^2 + \E_{x \sim p_n} \|  x- T (x) \|^2 ),
\end{align*}
where $ \E_{x \sim p_n} \| x \|^2  = M_2 (p_n) < \infty $, and, at the minimizer $T$, $\E_{x \sim p_n} \|  x- T (x) \|^2$ also needs to be finite
due to that it is in the 2nd term of  \eqref{eq:JKO-obj-1c}.
\end{proof}

\begin{proof}[Proof of Lemma \ref{lemma:Tp-in-P2r-new}]
We first show that $ T \in L^2(p)$ iff $T_\# p  \in \calP_2$.
This is because $M_2( T_\# p ) = \E_{x \sim p} \|T(x)\|^2$, and thus is finite iff. $T$ is in $L^2(p)$. 

As a result, when $T \in L^2(p)$, $T_\# p \in \calP_2$. 
If $T$ is also non-degenerate, Lemma \ref{lemma:NG-density} implies that $T_\# p$ also has density. 
This proves that $T_\# p \in \calP_2^r$.
\end{proof}

\section{Proofs and lemmas in Section \ref{sec:theory-foward}}\label{app:proof-sec-4}

\subsection{Technical lemmas in Section \ref{subsec:general-conditions}}\label{app:lemmas-sec-4.1}

\begin{lemma}\label{lem:H(rho)-convex-agg}
$\calH(\rho)$ is convex a.g.g. in $\calP_2$. 
\end{lemma}
\begin{proof}
The a.g.g.-convexity of functional in the form of 
$\calF(\rho) = \int F( \rho(x)) dx$ in $\calP_2$
is established in Proposition 9.3.9 of \cite{ambrosio2005gradient} when $F: [ 0, +\infty ) \to (-\infty, \infty]$ is a proper, lower semi-continuous convex function satisfying that 
$s \mapsto s^d F(s^{-d})$ is convex and non-increasing on $(0, +\infty)$.
The entropy $\calH(\rho) = \calF(\rho)$ with $F(s) = s\log s$, and this $F$ satisfies the above conditions. 
\end{proof}
\begin{lemma}
\label{lem:E(rho)-lambda-convex-agg}
Under Assumption \ref{assump:V-45}, $\calE(\rho)$ is    $\lambda$-convex a.g.g. in $\calP_2$. 
\end{lemma}

\begin{proof}
This is a direct result of Proposition 9.3.2(i) of \cite{ambrosio2005gradient}, noting that assuming the boundedness of $V^-$ implies the growth condition needed in Section 9.3 therein.
The proof of  Proposition 9.3.2(i) shows that $\calE(\rho)$ is $\lambda$-convex along any interpolation curve which implies $\lambda$-convexity a.g.g.
\end{proof}

\begin{lemma}\label{lemma:KL-satisfies-assump}
Under Assumptions \ref{assump:V-45}, 
the KL divergence $G(\rho)$ defined in \eqref{eq:def-KL-G} satisfies Assumption \ref{assump:genera-G}.
\end{lemma}
\begin{proof}
The lower semi-continuity follows from that of $\calH(\rho)$ and the condition on $V$ in Assumption \ref{assump:V-45}.
The domain of $G$ is restricted to $\rho$ with density because $\calH(\rho)$ diverges otherwise.
The a.g.g. $\lambda$-convexity of $G$
    directly follows from Lemma \ref{lem:H(rho)-convex-agg} and Lemma \ref{lem:E(rho)-lambda-convex-agg}.
\end{proof}

\subsection{Proofs in Section \ref{subsec:convergence-forward}}

\begin{proof}[Proof of Lemma \ref{lemma:mono-G}]
The unique existences of $T_p^\rho$ and $T_p^\pi$ are by Brenier Theorem.
Since $\rho \in \calP_2^r$, the map $T_p^\rho$ has an inverse denoted by $T_\rho^p$ which is defined $\rho$-a.e. 
Under Assumption \ref{assump:genera-G} first part, the strong subdifferential of $\partial_{\W} G (\rho)$  is well-defined, and we assume $\eta$ is one of them.

Let $v: = T_p^\pi\circ T_\rho^p - {\rm I_d}$.
One can verify that  $v \in L^2( \rho)$, 
since $\|T_p^\pi\circ T_\rho^p \|_\rho^2 = M_2(\pi)$,
$\| {\rm I_d} \|_\rho^2 = M_2(\rho)$,
and both are finite.
By definition, for $\delta \in [0,1]$, 
\begin{align}\label{eq:pf-mono-1}
    \pa{ {\rm I_d} +\delta v}_\# \rho 
& = ( {\rm I_d} +\delta (T_p^\pi\circ T_\rho^p - {\rm I_d}) )_\# \rho \notag
\\ & = ( T_p^\rho + \delta (T_p^\pi - T_p^\rho) )_\# p.
\end{align}
We also have 
\begin{align}\label{eq:pf-mono-2}
    \an{\eta, v }_\rho 
    = \an{\eta \circ {T_p^\rho}, T_p^\pi - T_p^\rho }_p. 
\end{align}

Since $v \in L^2( \rho)$, by that $\eta \in \partial_{\W} G( \rho)$ and the definition of strong subdifferential (Definition \ref{def:strong-subdiff}),
with $\delta \to 0+$ we have
\begin{align}\label{eq:xi_phi}
G \pa{\pa{ {\rm I_d} +\delta v}_\# \rho } 
\geq G( \rho ) + \delta\an{\eta, v }_\rho + o(\delta).
\end{align}
Combined with \eqref{eq:pf-mono-1}\eqref{eq:pf-mono-2}, this gives 
\begin{align}\label{eq:pf-mono-3}
& G \pa{ ( T_p^\rho + \delta (T_p^\pi - T_p^\rho) )_\# p }  - G( \rho ) \nonumber 
\\ &\geq \delta \an{\eta \circ {T_p^\rho}, T_p^\pi - T_p^\rho }_p + o(\delta).
\end{align}

Meanwhile, by the $\lambda$-convexity of $G$ a.g.g. (Definition \ref{def:convex-agg}), and specifically \eqref{eq:convex-agg-2}, we have
\begin{align}\label{eq:pf-mono-4}
& G \pa{ ( T_p^\rho + \delta (T_p^\pi - T_p^\rho) )_\# p } \nonumber \\ 
& \leq  (1-\delta) G( \rho )  + \delta G( \pi ) 
     - \frac{\lambda}{2}\delta(1-\delta)\W( \rho, \pi)^2.
\end{align}
Comparing \eqref{eq:pf-mono-3} and \eqref{eq:pf-mono-4}, we have 
\begin{align*}
    G(\pi) - G (\rho) 
    & \geq  
    \an{\eta \circ {T_p^\rho}, T_p^\pi - T_p^\rho }_p
    \\&\quad + \frac{\lambda}{2}(1-\delta)\W( \rho, \pi )^2 + o(1).
\end{align*}
We get the conclusion by letting $\delta \to 0+$.
\end{proof}

\begin{proof}[Proof of Lemma \ref{lemma:evi}]
In the $n$-th step, 
$T_{n+1}$ is in $L^2(p_n)$ and is non-degenerate under Assumption \ref{assump:1st-order-condition-error},
thus from $p_n \in \calP_2^r$,  
$p_{n+1} = (T_{n+1})_\# p_n$  is also in $\calP_2^r$ by Lemma \ref{lemma:Tp-in-P2r-new}.
This holds for $n=0,\cdots, N-1$, and thus all $p_n$ are in $\calP_2^r$,
where $p_0\in \calP_2^r$ is by the lemma assumption. 
By the Brenier Theorem, the OT map from $p_n$ to $p_{n+1}$ is denoted as $T_n^{n+1}$, which is uniquely defined $p_n$-a.e.
Let $T_{n+1}^n$ be the OT map from $p_{n+1}$ to $p_n$, and it is also the $p_{n+1}$-a.e. inverse of $T_{n}^{n+1}$.
We use the short-hand notation
\[
X_{n+1} := T_{n}^{n+1}.
\]

Under the assumption on $G$, Lemma \ref{lemma:ot_F} applies which gives the relationship between $\partial_{\W} F_{n+1}$ and $\partial_{\W} G$.
Together with the assumption on $\xi_{n+1}$ by Assumption \ref{assump:1st-order-condition-error}, we have that for each $n$, $\exists \eta_{n+1} \in \partial_{\W}G(p_{n+1})$ s.t.
\[
\gamma \xi_{n+1} - \gamma \eta_{n+1} =  {\rm I_d}  - T_{n+1}^n,  \quad p_{n+1}\text{-a.e.}
\]
and equivalently,
\begin{equation}\label{eq:pf-evi-1}
 {\rm I_d} - X_{n+1}  =  \gamma ( \eta_{n+1} -  \xi_{n+1} ) \circ X_{n+1}  . \quad p_{n}\text{-a.e.}
\end{equation}

Denote by $T_n^\pi$ the unique OT map from $p_n$ to $\pi$. 
Expanding $\ve{X_{n+1}-T_n^\pi}_{p_n}^2$ as 
\begin{align*}
&\quad\  \| X_{n+1} - T_n^\pi \|_{p_n}^2
\\& =  \| ( {\rm I_d} - X_{n+1}) - ( {\rm I_d} - T_n^\pi )\|_{p_n}^2 \\
& =  \|  {\rm I_d} - T_n^\pi \|_{p_n}^2 
	- 2 \langle {\rm I_d} - T_n^\pi  ,   {\rm I_d} - X_{n+1}  \rangle_{p_n}
	\\ &\quad + \|  {\rm I_d} - X_{n+1} \|_{p_n}^2  \\
& =  \|  {\rm I_d} - T_n^\pi \|_{p_n}^2 
	- 2 \langle X_{n+1} - T_n^\pi   ,   {\rm I_d} - X_{n+1}  \rangle_{p_n}
	\\ &\quad - \|  {\rm I_d} - X_{n+1} \|_{p_n}^2	 \\
& \le 	\|  {\rm I_d} - T_n^\pi \|_{p_n}^2 
	- 2 \langle X_{n+1} - T_n^\pi   ,   {\rm I_d} - X_{n+1}  \rangle_{p_n},
\end{align*}
where in the last inequality we use that $\|  {\rm I_d} - X_{n+1} \|_{p_n}^2 \ge 0$.
By that 
\[
\|  {\rm I_d} - T_n^\pi \|_{p_n}^2  = \W( p_n, \pi)^2, 
\]
and together with \eqref{eq:pf-evi-1},  
we have 
\begin{align}
 &   \| X_{n+1} - T_n^\pi \|_{p_n}^2 
 \le  \W( p_n, \pi)^2 
  \nonumber
 \\ &
 ~~~
 - 2  \gamma \langle X_{n+1} - T_n^\pi   ,   ( \eta_{n+1} -  \xi_{n+1} ) \circ X_{n+1}   \rangle_{p_n}.\label{eq:pf-evi-3}
\end{align}

Applying Lemma \ref{lemma:mono-G} with $p = p_n$ and $\rho = p_{n+1}$, we have 
\begin{align}
& G(\pi) - G(p_{n+1}) \nonumber \\ & \geq  
	\langle T_n^\pi - X_{n+1} ,  \eta_{n+1} \circ X_{n+1} \rangle_{p_n} 
	+ \frac{\lambda}{2}\W(p_{n+1}, \pi )^2. \label{eq:pf-evi-4}
\end{align}
Meanwhile, by Cauchy Schwartz,
\begin{align*}
& | \langle X_{n+1} - T_n^\pi   ,   \xi_{n+1}  \circ X_{n+1}   \rangle_{p_n} |\\
& \le  \| X_{n+1} - T_n^\pi  \|_{p_n}  \| \xi_{n+1}  \circ X_{n+1} \|_{p_n}  \nonumber \\
& \le \varepsilon \| X_{n+1} - T_n^\pi  \|_{p_n} \nonumber 
\end{align*}
where the 2nd inequality is by that 
$\| \xi_{n+1}  \circ X_{n+1} \|_{p_n} = \| \xi_{n+1}  \|_{p_{n+1}}  \le \varepsilon $ 
(Assumption \ref{assump:1st-order-condition-error}).
Since $\lambda > 0$, we have 
\begin{equation}\label{eq:pf-evi-CS}
\varepsilon \| X_{n+1} - T_n^\pi  \|_{p_n} 
\le \frac{\varepsilon^2}{\lambda} + \frac{\lambda}{4} \| X_{n+1} - T_n^\pi  \|_{p_n}^2.
\end{equation}
Putting together, this gives 
\begin{equation} \label{eq:pf-evi-5}
| \langle X_{n+1} - T_n^\pi   ,   \xi_{n+1}  \circ X_{n+1}   \rangle_{p_n} | 
 \le\frac{\varepsilon^2}{\lambda} + \frac{\lambda}{4} \| X_{n+1} - T_n^\pi  \|_{p_n}^2.
\end{equation}
Inserting \eqref{eq:pf-evi-4}\eqref{eq:pf-evi-5} into \eqref{eq:pf-evi-3} gives 
\begin{align}\label{eq:pf-evi-6}
& \ (1-\frac{ \gamma \lambda}{2} ) \| X_{n+1} - T_n^\pi \|_{p_n}^2\nonumber \\
   &\le  \W( p_n, \pi)^2\nonumber	+ 2  \gamma 
	   \left( G(\pi) - G(p_{n+1}) - \frac{\lambda}{2}\W(p_{n+1}, \pi )^2 \right)
 	\nonumber \\ & \quad +  \frac{2 \gamma}{\lambda}\varepsilon^2.
\end{align}
Because $( X_{n+1}, T_n^\pi )_\# p_{n}$ is a coupling between $p_{n+1}$ and $\pi$, we have
\begin{equation}\label{eq:pf-evi-2}
    \W (p_{n+1}, \pi)^2  \leq \ve{X_{n+1}-T_n^\pi}_{p_n}^2.
\end{equation}
Under the condition of the lemma, $ 0 <\gamma  \lambda  < 2$, 
and thus $1-\frac{ \gamma \lambda}{2}  >0$ and
then the l.h.s. of \eqref{eq:pf-evi-6} $\ge  (1-\frac{ \gamma \lambda}{2} )  \W(p_{n+1}, \pi )^2$.
This proves \eqref{eq:lemma-evi}.
\end{proof}

\begin{proof}[Proof of Theorem \ref{thm:N-step-forward-no-inv-error}]
Taking $\pi = q$ and apply Lemma~\ref{lemma:evi}, by that 
$2 \gamma \pa{ G(p_{n+1}) - G(\pi)} \ge 0$, 
\eqref{eq:lemma-evi} gives that for all $n$, 
\begin{equation}\label{eq:pf-thm-convegence-1}
   \left( 1+\frac{ \gamma \lambda }{2}\right) \W^2(p_{n+1}, q )
   \leq 
   \W^2(p_n, q ) 
   + \frac{2 \gamma }{\lambda} \varepsilon^2.
\end{equation}
Define the numbers $\rho$ and $\alpha$ as 
\[
\rho: =   \left( 1+ \frac{ \gamma \lambda }{2}\right)^{-1}, \quad 0 < \rho < 1,
\quad \alpha := \sqrt{\frac{2 \gamma }{\lambda}} \varepsilon,
\]
and define
\[
E_n: = \W(p_{n}, q )^2, 
\]
then \eqref{eq:pf-thm-convegence-1} can be written as
\[
 E_{n+1} \le \rho( E_n + \alpha).
\]
Recursively applying from 0 to $n-1$ gives that 
\[
E_n 
\le \rho^n E_0 + \alpha \frac{\rho (1-\rho^n)}{1-\rho}
\le\rho^n E_0 + \alpha \frac{\rho}{1-\rho},
\]
which by definition is equivalent to \eqref{eq:thm-exp-rate}.

By \eqref{eq:thm-exp-rate}, one will have $\W(p_{n}, q )^2 \le 5 \varepsilon^2/\lambda^2$ if 
\[
 \left( 1+ \frac{\gamma \lambda}{2} \right)^{-n} \W^2(p_0, q)  \le  \frac{\varepsilon^2}{\lambda^2},
\]
which is fulfilled as long as
\[
n \ge  \frac{2 \left( \log  \W(p_0, q)  + \log ({\lambda}/{\varepsilon})  \right)}{\log( 1+ \gamma \lambda/2)}.
\]
This requirement of $n$ is satisfied under \eqref{eq:thm-exp-n-threshold} by that $ 0 < \gamma \lambda < 2$ 
and the elementary relation that 
$\log (1+x) \ge x/2$ for $x \in (0,1)$.
We have proved the $\W$-error bound.

To show the smallness of the objective gap $G(p_n) - G(q)$, we use \eqref{eq:lemma-evi} again,
 and by that 
$\W^2(p_{n+1}, q )  \ge 0$, 
\begin{equation}
  2\gamma \pa{ G(p_{n+1}) - G(q)}
   \leq 
   \W^2(p_n, q) 
   + \frac{2 \gamma }{\lambda} \varepsilon^2.
\end{equation}
When $n$ already makes $\W(p_{n}, q )^2 \le 5 \varepsilon^2/\lambda^2$, we have
\begin{equation}
  2\gamma \pa{ G(p_{n+1}) - G(q)}
   \leq 
   ( 5 + 2 \gamma \lambda  ) \frac{\varepsilon^2}{\lambda^2}
   \leq 
   9 \frac{\varepsilon^2}{\lambda^2},
\end{equation}
where in the 2nd inequality we use that $\gamma \lambda < 2 $
because $ 0< \lambda \le 1$ and $ 0 < \gamma < 2$.
This proves the bound of $G(p_n) - G(q)$ in \eqref{eq:thm-exp-bound}.
\end{proof}

\section{Proofs in Section \ref{sec:theory-reverse}}\label{app:proofs-sec5}

\subsection{Proofs in Section \ref{subsubsec:bwd-no-inv-P2r}}

\begin{proof}[Proof of Lemma \ref{lemma:sym-KL}]
Let $X_1 \sim p$, $X_2 \sim q$, and 
\[
Y_1 = T(X_1), \quad Y_2 = T(X_2).
\]
Then $Y_1$ and $Y_2$ also have densities, $Y_1 \sim \tilde p:= T_\# p$ and $Y_2 \sim \tilde q := T_\# q$.
By the data processing inequality 
concerning two probability distributions through the same stochastic transformation for the KL divergence
(see, e.g., the introduction of \cite{raginsky2016strong}),
\[
{\rm KL}( \tilde p || \tilde q) \le {\rm KL}( p || q).
\]
In the other direction, $X_i = T^{-1} (Y_i)$, $i=1,2$, then data processing inequality also implies 
\[
{\rm KL}( p || q) \le {\rm KL}( \tilde p || \tilde q).
\]
\end{proof}

\begin{proof}[Proof of Corollary \ref{cor:KL-P-P2r}]
Under Assumption \ref{assump:V-45}, 
the KL divergence $G(\rho)$ satisfies Assumption \ref{assump:genera-G} (Lemma \ref{lemma:KL-satisfies-assump}),
also $q \in \calP_2^r$ and $G(q) = 0$ is the global minimum of $G$. 
The needed assumptions of Theorem \ref{thm:N-step-forward-no-inv-error} are all satisfied,
by which we have that
for the $N$ defined in the corollary, 
\begin{equation}\label{eq:bound-GpN-pf}
 {\rm KL}( p_N || q)  = 
G(p_{N}) \le  \frac{9}{2 \gamma } \left(\frac{\varepsilon}{\lambda} \right)^2.
\end{equation}

Under Assumption \ref{assump:1st-order-condition-error}, 
$T_n$ are all invertible, and thus
$T_1^N$ as defined in \eqref{eq:def-T1N} is invertible.
In addition, by Definition \ref{eq:def-ND}, one can verify that if $T_1$ and $T_2$ are non-degenerate, then so is $T_2\circ T_1$. 
Using the arguments for $N-1$ times, we have that $T_1^N$ is non-degenerate. 
Similarly, since each $T_n^{-1}$ is non-degenerate, we have that $(T_1^N)^{-1}$ is also non-degenerate.
Now, $p_0 = p $ has density, and then $p_N = (T_1^N)_\# p_0$ also has density (Lemma \ref{lemma:NG-density}).
Meanwhile, $q_N = q$ has density (Assumption \ref{assump:V-45}), then $q_0 = ((T_1^N)^{-1})_\# q_N$ also has density.
We now have that 
$p_0$, $q_0$, $p_N = (T_1^N)_\# p_0$ and $q_N = (T_1^N)_\# q_0$ all have densities.
Then Lemma \ref{lemma:sym-KL} gives that 
\[
 {\rm KL}( p_0 || q_0) = {\rm KL}( p_N || q_N)  ={\rm KL}( p_N || q), 
 \] 
 which, by \eqref{eq:bound-GpN-pf}, is bounded as stated in the corollary.
The TV bound is followed by Pinsker's inequality.
\end{proof}

\subsection{Proofs in Section \ref{subsubsec:P-P2-short-diffusion}}

\begin{lemma}[$\rho_\delta$ and $\W$ closeness]\label{lemma:rho-delta-W2}
Suppose $P \in \calP_2$, 
and $\rho_t$ is the density of $X_t$ in an OU process as in \eqref{eq:OU-SDE},
then 

(i) $\rho_t \in \calP_2^r$ for any $t>0$, 

(ii)  $\forall \varepsilon > 0$, 
$\exists \delta > 0$ s.t. $\W(\rho_\delta, P) <\varepsilon$.
In this case, one can choose $\delta \sim \varepsilon^2$.
\end{lemma}
\begin{proof}[Proof of Lemma \ref{lemma:rho-delta-W2}]
For the OU process, we have $V(x) = \|x\|^2/2$ in \eqref{eq:diffusion-sde-2}.
Then for any $t>0$, $\rho_t = {\calL}_t (P)$ has the expression as, with $\sigma_t^2 := 1-e^{-2 t}$,  
\begin{equation}
\rho_t(x) =  \int_{\R^d}  \frac{1}{(2\pi \sigma_t^2)^{d/2}} e^{- \| x - e^{-t} y \|^2/(2 \sigma_t^2)}  dP(y).
\end{equation}
Equivalently, $\rho_t$ is the probability density of the random vector
\[
Z_t: = e^{-t} X_0 + \sigma_t Z, \quad Z \sim \calN(0,I_d), 
\]
where  $Z$ is independent from $X_0$.
Since $\E \| Z_t \|^2 = e^{-2t} M_2(P) + \sigma_t^2 d < \infty$, we have $\rho_t \in \calP_2$ and this proves (i).

To prove (ii):
Because the law of $(Z_t, X_0)$ is a coupling of $\rho_t$ and $P$,
\begin{align*}
\W(\rho_t, P )^2
& \le \E \| Z_t - X_0  \|^2 \\
& = \E \| (e^{-t} -1)X_0 + \sigma_t Z  \|^2 \\
& = (1- e^{-t} )^2 M_2(P) +  (1-e^{-2t}) d \\
& \le t^2 M_2(P) + 2t  d, 
\end{align*}
where in the last inequality we used that 
$1- e^{-x} \le x$, $\forall x \ge 0$.
Since $M_2(P)  < \infty$, we have bounded $\W(\rho_t, P )^2$ to be $O(t)$.
\end{proof}

\begin{proof}[Proof of Corollary \ref{cor:KL-mixed-P2}]
For the $\varepsilon$ in Assumption \ref{assump:1st-order-condition-error},
the existence of $\delta$ to make $\W(\rho_\delta, P) <\varepsilon$ is by Lemma \ref{lemma:rho-delta-W2},
and we also have $\rho_\delta \in \calP_2^r$.
The rest of the proof is the same as in Corollary \ref{cor:KL-P-P2r} by starting from $p_0 = \rho_\delta$.
\end{proof}

\subsection{Proofs and lemmas in Section \ref{subsec:theory-with-inv}}\label{app:proofs-sec5.2}

\begin{lemma}[Lipschitz bound of ODE solution map]
\label{lemma:ode-lip}
Suppose for $\gamma > 0$,
 $\hat{v}(x,t)$ is $C^1$ in $(x,t)$ 
  and Lipschitz in $x$ uniformly on $\R^d \times [0, \gamma]$  with Lipschitz constant $K \ge 0$.
 Let $x(t)$ be the solution to the ODE
 \begin{equation}\label{eq:ode-0-gamma}
 \dot{x}(t) = \hat v( x(t), t), \quad t \in [0,\gamma], 
\end{equation}
and  define the solution map from 0 to $\gamma$ as $T: \R^d \to \R^d$, that is,
\begin{equation}
    T(x_0) = x_0 + \int_{0}^{\gamma} \hat v( x(t), t) dt, \quad x(0) = x_0.
\end{equation}
Then $T$ is invertible on $\R^d$,
and both $T$ and $T^{-1}$ are Lipschitz on $\R^d$ with Lipschitz constant $e^{ \gamma K}$.
\end{lemma}

\begin{proof}[Proof of Lemma \ref{lemma:ode-lip}]
Let $x_1(t)$ and $x_2(t)$ be the solution to the ODE \eqref{eq:ode-0-gamma}
from $x_1(0) = y$, and $x_2(0) = z$ respectively. 
By definition, 
\[
T(y) = x_1( \gamma), \quad T(z) = x_2( \gamma).
\]
Under the condition of $\hat v$, the ODE is well-posed \cite{Sideris2013OrdinaryDE}.
This implies the invertibility of $T$, and $T^{-1}$ is the solution map of the reverse time ODE from $t= \gamma$ to $t=0$. 

We now prove the Lipschitz constant of $T$ on $\R^d$, and that of $T^{-1}$ can be proved similarly by considering the reverse time ODE.
We want to show that 
\[
\| T(y) - T(z) \|  \le e^{\gamma K} \| y - z \|, \quad \forall y, z \in \R^d,
\]
and this is equivalent to that for any $x_1(0), \, x_2(0) \in \R^d$, 
\begin{equation}\label{eq:pf-Lip-ode-1}
\| x_1 (\gamma) - x_2 (\gamma ) \|  \le e^{\gamma K}  \|  x_1(0) - x_2 (0) \|.
\end{equation}
For fixed $x_1(0), \, x_2(0)$, define 
\[
E(t) := \frac{1}{2} \|  x_1(t) - x_2(t)\|^2, 
\]
then $E(0) =  \|  x_1(0) - x_2 (0) \|^2/2$, and
\begin{align*}
\dot{E}(t) 
& = (x_1(t) - x_2(t))^T ( \dot{x}_1(t) - \dot{x}_2(t) ) \\
& = (x_1(t) - x_2(t))^T ( \hat v({x}_1(t),t) - \hat v( {x}_2(t), t) ).
\end{align*}
Thus, by that $\|  \hat v({x}_1(t),t) - \hat v( {x}_2(t), t) \| \le K  \| x_1(t) - x_2(t) \|$, we have
\[
\dot{E}(t)  
\le K \| x_1(t) - x_2(t) \|^2  =2  K E(t).
\]
By Gr\"onwall's inequality, $E(t) \le E(0) e^{ 2K t}$, and this gives
\[
 \|  x_1(t) - x_2(t)\|^2 \le e^{2K t} \|  x_1(0) - x_2 (0) \|^2, \quad t \in [0,\gamma].
\]
Setting $t = \gamma $ proves \eqref{eq:pf-Lip-ode-1}.
\end{proof}

\begin{lemma}\label{lemma:Lip-is-L2}
Suppose $p \in \calP_2^r$, and $T:\R^d \to \R^d$ is Lipschitz on $\R^d$,
then $T \in L^2(p)$. 
\end{lemma}
\begin{proof}
We are to show that
\[
 \E_{x \sim p} \| T (x) \|^2 < \infty.
\]
Suppose $T$ is $L$-Lipschitz on $\R^d$, then $\forall x\in \R^d$, 
\[
\| T(x) \| 
\le \| T(0) \| + \|T(x) - T(0) \|
\le \| T(0) \| + L \|x \|.
\]
Thus,
\begin{align*}
 \E_{x \sim p} \| T (x) \|^2
 & \le 2(   \| T(0) \|^2 + L^2 \E_{x \sim p} \|x \|^2)
 \\ &= 2(   \| T(0) \|^2 + L^2 M_2 (p) ) < \infty,   
\end{align*}
because $M_2(p) <\infty$. 
\end{proof}

\begin{proof}[Proof of Proposition \ref{prop:w2-q0}]
By construction, for $n=1, \cdots, N$, 
\[
q_{n-1} = (T_n^{-1} )_\# q_n, \quad
\tilde q_{n-1} = (S_n )_\# \tilde q_n.
\]
We also have that $q_N = \tilde q_N = q \in \calP_2^r$ by assumption.

For the sequence of $q_n$, we know that $T_n^{-1}$ is non-degenerate (Assumption \ref{assump:1st-order-condition-error}).
Meanwhile, for $q_n \in \calP_2^r$, $T_n^{-1}$ being globally Lipschitz on $\R^d$ also implies that it is in $L^2(q_n)$ (Lemma \ref{lemma:Lip-is-L2}).
Then by Lemma \ref{lemma:Tp-in-P2r-new}, $q_{n-1} = (T_n^{-1} )_\# q_n$ is also in $\calP_2^r$.
For the sequence of $\tilde q_n$, by Assumption \ref{assump:inv-error}, $S_n$ is non-degenerate and in $ L^2( \tilde q_n)$, 
thus from $\tilde q_n \in \calP_2^r$, $\tilde q_{n-1}= (S_n )_\# \tilde q_n$ is also in $\calP_2^r$ by Lemma \ref{lemma:Tp-in-P2r-new} again.
Thus by induction, we have that $q_n$ and $\tilde q_n$ are all in $\calP_2^r$.

For each $n$, we have
\begin{align*}
    \W(\tilde q_{n-1}, q_{n-1}) 
    & = \W ((S_n )_\# \tilde q_n,  (T_n^{-1} )_\# q_n )\\
    & \leq \underbrace{\W ((S_n )_\#  \tilde q_n, (T_n^{-1} )_\# \tilde q_n)}_{\textcircled{1}} \\ & + \underbrace{\W ((T_n^{-1}  )_\# \tilde q_n, (T_n^{-1}  )_\#  q_n )}_{\textcircled{2}}.
\end{align*}
To bound $\textcircled{1}$, we use Assumption \ref{assump:inv-error} and the Lipschitzness of $T_n^{-1}$. 
Define $L :=e^{\gamma K}$.
Using $(S_n, T_n^{-1} )_\# \tilde q_n$ as the coupling, we have
\begin{align*}
& \W^2 ((S_n )_\#  \tilde q_n, (T_n^{-1} )_\#  \tilde  q_n) \\
& \le \int_{\R^d} \| S_n(x) - T_n^{-1}(x) \|^2  \tilde  q_n(x) dx \\
& \le \int_{\R^d} L^2 \| T_n \circ S_n(x) -x \|^2  \tilde  q_n(x) dx  
\\ 
& = L^2 \| T_n \circ S_n - {\rm I_d} \|_{\tilde q_n}^2,
\end{align*}
where for the second inequality, we use the fact that $T_n^{-1}$ is $L$-Lipschitz. Thus, 
\begin{equation}
\textcircled{1} \le L  \varepsilon_{\rm inv}.
\end{equation}
To bound $\textcircled{2}$, we use that $T_n^{-1}$ is $L$-Lipschitz on $\R^d$ again.
Specifically, let $Y_n$ be the unique OT  map from $q_n$ to $\tilde q_n$ which is well-defined by the Brenier Theorem,
then $(T_n^{-1} \circ Y_n, T_n^{-1} )_\# q_n$ is a coupling of $(T_n^{-1} )_\# \tilde q_n$ and $(T_n^{-1})_\#  q_n$. 
We have that 
\begin{align*}
& \W^2(( T_n^{-1} )_\# \tilde q_n, ( T_n^{-1} )_\#  q_n) \\
& \le  \int_{\R^d} \| T_n^{-1} \circ Y_n (x) - T_n^{-1} (x)  \|^2 q_n(x) dx \\
& \le  \int_{\R^d} L^2 \|  Y_n (x) - x  \|^2 q_n(x) dx  \\
& = L^2  \W^2( \tilde q_n ,q_n).
\end{align*}
Thus 
\begin{equation}
\textcircled{2} \le L  \W( \tilde q_n ,q_n).
\end{equation}
Putting together, we have
\[
 \W(\tilde q_{n-1}, q_{n-1})  \le e^{\gamma K } ( \varepsilon_{\rm inv} +  \W( \tilde q_n ,q_n)).
\]
Note that $\W( \tilde q_N ,q_N) = 0$ by that $\tilde q_N = q_N = q$.
Applying recursively from $n= N$ to $n=1$ gives that
\begin{equation}
\W( \tilde q_0, q_0 ) \le \varepsilon_{\rm inv} \frac{ e^{\gamma K }  (e^{\gamma K N} -1 )}{ e^{\gamma K } -1},
\end{equation}
which proves \eqref{eq:W-tilq0-q0} by that $e^x -1 \ge x$ for any $x \in \R$. 
\end{proof}

\begin{proof}[Proof of Corollary \ref{cor:mixed-bound-inv}]
Under the condition of the corollary,
Corollary \ref{cor:KL-P-P2r} applies to bound ${\rm KL}( p || q_0)$ and $TV ( p, q_0)$ as in \eqref{eq:bound-cor-KL-P-Pr2},
and 
Proposition \ref{prop:w2-q0} applies to bound $\W( \tilde q_0, q_0) $ as in \eqref{eq:W-tilq0-q0}.
It suffices to show that the r.h.s. of \eqref{eq:W-tilq0-q0} is less than or equal to that of \eqref{eq:w2-tidleq0-q0-bound-final}.

By the choice of $N$ in  \eqref{eq:def-big-N},
\[
N  \le  \frac{8 }{ \gamma \lambda} \left( \log  \W(p_0, q)  + \log ({\lambda}/{\varepsilon})  \right)  +1,
\]
and thus
\[
e^{ \gamma K (N+1) } \le  e^{2\gamma K}    \left(\W(p_0, q)  \frac{ \lambda}{\varepsilon} \right)^{{8 K}/{ \lambda} }   ,
\]
which proves the needed inequality.
\end{proof}

\section{Numerical evidence to support Assumption \ref{assump:1st-order-condition-error}}\label{app:exp}

\begin{figure}[t]
\centering
\includegraphics[height=.35\linewidth]{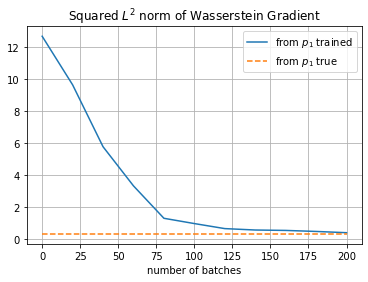} 
\caption{
Computed values of $\| \nabla_{\W} F_{n+1} (p_{n+1}) \|_{p_{n+1}}^2 $ from $N=2000$ samples, where $p_{n+1}$ is pushforwarded by a trained neural network transport $T$ from a Gaussian initial $p_n$ in $\R^2$, $n=0$. 
The blue line shows the value as the training progresses, and the dashed line is a base value computed from the analytical solution $p_{n+1}^{\rm true}$ (where the Wasserstein gradient vanishes). 
}
\label{fig:wass-2d-exp}
\end{figure}

\begin{figure*}[t]
\centering
\includegraphics[height=.5\linewidth]{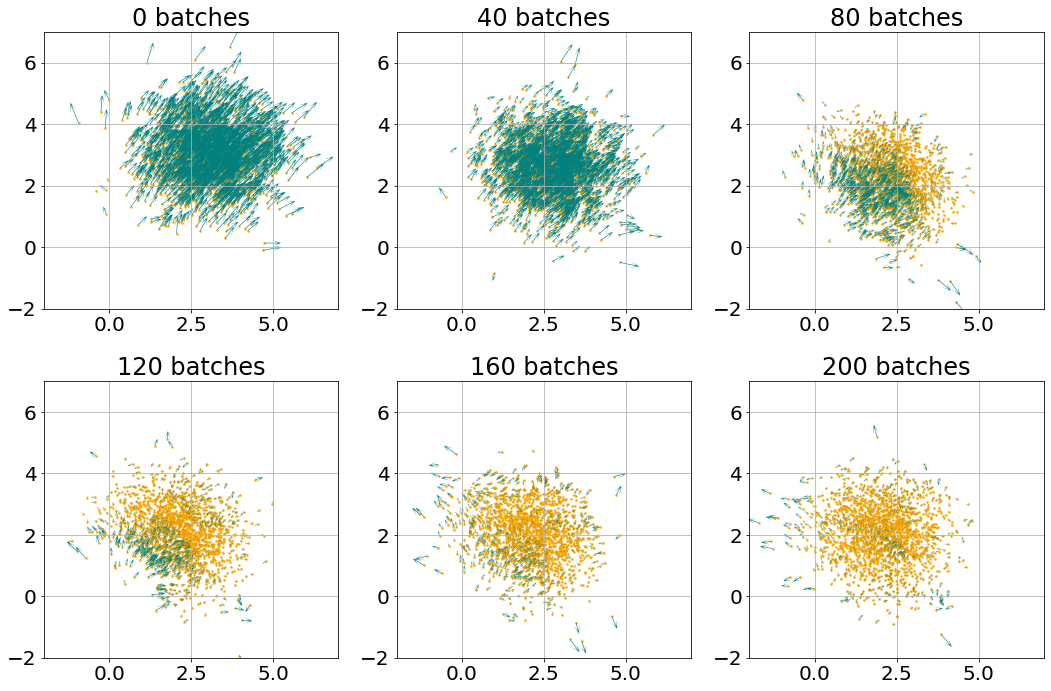} 
\caption{
The Wasserstein gradient vector field $\xi$ at samples $x_i^{(1)} = T(x_i)$ (shown by green arrows), where $T$ is the trained neural network transport map, plotted as the training progresses. 
The yellow dots are samples $x_i^{(1)}$.
The length of the arrow is proportional to the magnitude of $\| \xi(x_i^{(1)})\|$. 
}
\label{fig:vec_field-exp}
\end{figure*}

We conduct training of one JKO block to verify \eqref{eq:A1-small-|xi_n|} in Assumption \ref{assump:1st-order-condition-error}.
The code is available at \url{https://github.com/yixintan-zeta/jko_wass_grad}.

Data is in $\R^2$.
Let $n=0$, $p_0  = \calN( [3,3]^T, I)$, 
$q = \calN(0, I)$, and $G(\rho) = {\rm KL}(\rho || q)$.
The step size $\gamma = 0.5$.
We train a JKO flow network block to minimize $F_{n+1}$ in \eqref{eq:JKO-obj-1b} via the training objective \eqref{eq:JKO-obj-1c}, where $T$ is parametrized by a neural ODE block consisting of two hidden layers with 128 hidden dimensions and using the softplus activation ($\beta =20$).
We use 10,000 training samples, batch size 2000, and 200 total iterations, with learning rate $10^{-4}$.
The implementation of the JKO flow network follows the setup in \cite{xu2022jko}.

Making use of the explicit expression of $\nabla_{\W} G(\rho)$, let $\xi $ be $\xi_1$, then \eqref{eq:expression-W2-grad-Fn+1} gives that 
\[
\xi = (\nabla V + \nabla \log p_1)  -  \frac{T_{p_1}^{p_0} - {\rm I_d}}{\gamma},
\]
where $V(x) = \|x\|^2/2$. 
To compute $\xi$ as a vector field, we used $N=2000$ samples $x_i \sim p_0$, and then the neural-network trained $T$ will push-forward $x_i$ to $T(x_i)  \sim p_1$. Let $X_0 = \{ x_i^{(0)}:= x_i\}_{i=1}^N$ and $X_1 = \{ x_i^{(1)}:= T(x_i)\}_{i=1}^N$ be two data clouds. 
We approximate the OT map $T_{p_1}^{p_0}$ evaluated on $x_i^{(i)}$ by the solution of a discrete OT problem, computed by the Python POT package \cite{POTlink}. 
The score function $\nabla \log p_1$ is approximately computed by a kernel approach, where we used a Gaussian kernel with a properly chosen bandwidth parameter. 
Once $\xi(x_i^{(1)})$ is computed at every sample, we can approximately compute the squared $L^2$ norm $\| \xi \|_{p_1}^2$ by a sample average. 

We compute $\xi$ at $p_1$ induced by trained $T$ not only at the end of training but also during intermediate iterations. This will show the change in the Wasserstein gradient as the neural network training progresses. 
At 0, 20, ..., 200 batches, the estimated $L^2$ norm is shown in Figure \ref{fig:wass-2d-exp}. The dash-line shows the computed value of the Wasserstein gradient at the true solution $p_1^{\rm true} = \calN( [2,2]^T,I )$, since the JKO step is from a Gaussian $p_0$ and thus the true population minimizer $p_1$ is analytically available. The numerical value is not exactly zero because it is also computed on $N=2000$ finite samples. The dash-line shows a baseline of the numerical $L^2$ norm, and it can be seen that at the end of training, the neural network learned $p_1$ archives a comparable value.

We further illustrate the vector field $\xi$ on $x_i^{(1)}$'s in Figure \ref{fig:vec_field-exp}, which shows the evolution of the Wasserstein gradient over training iterations. It can be seen that the magnitude of the vector field decreases, and the values are getting small at least within the region where the distribution density $p_1$ has a large value. 
In the outskirts, the vector field $\xi$ does not numerically get small because it is at the tail of the Gaussian distribution,
where a small $L^2$ norm does not imply pointwise smallness of $\xi$ in these regions
and the estimation of the vector field may also be of lower accuracy.

\end{document}